\documentclass[journal]{IEEEtran}
\usepackage{amsmath,amsfonts}
\usepackage{verbatim}
\usepackage{adjustbox}
\usepackage[ruled,linesnumbered]{algorithm2e}
\usepackage{array}
\usepackage[caption=false,font=normalsize,labelfont=sf,textfont=sf]{subfig}
\usepackage{textcomp}
\usepackage{stfloats}
\usepackage{verbatim}
\usepackage{graphicx}
\usepackage{multirow}
\usepackage{url}
\usepackage{graphicx}
\usepackage{amssymb}
\usepackage{booktabs}
\usepackage{comment}
\usepackage{times}
\usepackage{epsfig}
\usepackage{colortbl}
\usepackage{xcolor}
\usepackage{bm} 
\usepackage{amsmath} 
\usepackage{amsthm}
\usepackage{tabularx}

\usepackage{xcolor}

\definecolor{lightblue}{RGB}{217, 237, 247}
\definecolor{lightgreen}{RGB}{229, 245, 224} 

\newcommand{\rval}[2]{{#1\,\textsubscript{(#2)}}}

\usepackage{soul}

\usepackage[numbers,sort&compress]{natbib}

\usepackage[pdftex,colorlinks=true,citecolor=blue,linkcolor=blue,urlcolor=black]{hyperref}

\usepackage[capitalize,noabbrev]{cleveref}

\newtheorem{theorem}{Theorem}
\newtheorem{remark}{Remark}
\newtheorem{corollary}{Corollary}[theorem]
\newtheorem{lemma}{Lemma}
\newtheorem{definition}{Definition}
\newtheorem{assumption}{Assumption}

\crefname{figure}{Fig.}{Figs.}
\crefname{table}{Table}{Tables}
\crefname{section}{Section}{Sections}
\crefname{theorem}{Theorem}{Theorems}
\crefname{lemma}{Lemma}{Lemmas}
\crefname{definition}{Definition}{Definitions}
\crefname{equation}{Equation}{Equations}
\crefname{assumption}{Assumption}{Assumptions}
\crefname{algorithm}{Algorithm}{Algorithms}
\crefname{corollary}{Corollary}{Corollaries}

\begin{document}

\title{Efficient  
  Temporal Point Processes via Monotone Alternating Splines}
\author{Cheng~Wan, Quyu~Kong, Feng~Zhou~\IEEEmembership{Senior Member,~IEEE}
\IEEEcompsocitemizethanks{
\IEEEcompsocthanksitem Cheng Wan is with Hong Kong University of Science and Technology, Hong Kong. Email: wanc7147@gmail.com. 
\IEEEcompsocthanksitem Quyu Kong is with Alibaba Cloud, Hangzhou, China. Email: kongquyu@gmail.com;
\IEEEcompsocthanksitem Feng Zhou is with Center for Applied Statistics and School of Statistics, Renmin University of China, Beijing, China. Email: feng.zhou@ruc.edu.cn. 
}
}


\markboth{preprint}
{Wan \MakeLowercase{\textit{et al.}}: Efficient Temporal Point Processes via Monotone
Alternating Splines}


\maketitle

\begin{abstract}
Temporal point processes (TPPs) have widespread applications across various domains. Compared to modeling the conditional intensity of a TPP, modeling its cumulative conditional intensity function (CCIF) improves computational efficiency and eliminates numerical approximation errors. However, current CCIF parameterizations uniformly rely on Monotone Neural Networks (MNNs), which we identify as suffering from three structural deadlocks--convexity restrictions, saturation limits, and violations of CCIF modeling requirements--that fundamentally restrict their representational capacity for complex temporal dynamics. To resolve these bottlenecks, this paper proposes a novel framework called Monotone Alternating Splines (MAS). By leveraging distinct interpolation and extrapolation components, MAS provides a flexible and efficient framework for modeling CCIFs. Theoretically, MAS's interpolation provides strong fitting accuracy, while its extrapolation supports robust generalization, reducing the irreducible approximation gaps of MNNs. Extensive experiments show that MAS achieves superior performance on both synthetic and real-world datasets.

\end{abstract}


\section{Introduction} \label{intro}

Event sequences are prevalent across various domains, such as criminology~\citep{mohler2011self,zhou2020auxiliary}, social networks~\citep{chen2018marked,meng2024interpretable}, finance~\citep{bacry2015hawkes,hawkes2018hawkes}, seismology~\citep{ogata1998space,ogata1999seismicity}, and neuroscience~\citep{linderman2016bayesian,zhou2022efficient}. These sequences are typically characterized by irregular timestamps in continuous time, which sets them apart from conventional regular time series data. Understanding the dynamics within these sequences provides valuable insights for forecasting future events. Temporal point processes (TPPs) offer a statistical framework for modeling such event sequences. 

The key to characterizing a TPP is the next timestamp's distribution given its history. 
One of the most common approaches to parameterize a TPP is to specify a conditional intensity function (CIF). 
Numerous methods \citep{du2016recurrent,mei2017neural,simiao2020transformer} leverage deep learning approaches to model the CIF and recover the model parameters via maximum likelihood estimation (MLE). 
However, CIF-based paradigms suffer from a persistent computational bottleneck. Evaluating the closed-form log-likelihood of a TPP sequence requires integrating the CIF over a time window. If the CIF is based on complicated neural networks, the integral must be approximated by numerical integration, such as Monte Carlo \citep{hastings1970monte}. 
This introduces cumulative errors and severely limits computational efficiency during both training and inference. 

To entirely bypass integration, an equivalent alternative is to directly model the Cumulative Conditional Intensity Function (CCIF)~\citep{omi2019fully}. By definition, the CIF is the exact derivative of the CCIF. Consequently, the required CIF can be recovered via automatic differentiation~\citep{baydin2018automatic}. CCIF-based parameterization provides an exact, integral-free paradigm that eliminates integration errors and accelerates likelihood evaluations.

Despite its structural elegance, the current CCIF parameterization is severely constrained. 
Because the underlying CIF is strictly non-negative, a valid CCIF must mathematically be a continuous, monotonically increasing function. Current literature uniformly enforces this constraint using Monotone Neural Networks (MNNs)~\citep{omi2019fully, liu2024cumulative}. Typically, MNNs are implemented using Multi-layer Perceptrons (MLPs) with strictly positive weights and non-decreasing activation functions. 
However, in this paper, we identify that the MNN architecture inevitably encounters three structural deadlocks: 
(1) \textbf{The Convexity Restriction:} Using purely convex activations (e.g., ReLU, Softplus) restricts the CCIF to be strictly convex, fundamentally preventing the predicted CIF from decaying after excitation. 
(2) \textbf{The Saturation Limit:} Even when bounded, non-convex activations (e.g., Sigmoid or Tanh) are used instead of convex ones in the MNN, the network faces a different issue: the predicted CCIF will stop growing and eventually stagnate. Consequently, the network becomes incapable of modeling long-horizon point processes where events persistently occur over time.
(3) \textbf{Violations of CCIF Requirements:} Finally, to forcibly bypass the convexity of MNNs, some methods apply non-convex temporal transformations in advance, such as $t\rightarrow\log(t)$. However, such transformations actually violate the fundamental mathematical requirements of the CCIF. Additionally, they heuristically force the intensity function to decay over time, which severely limits the model's flexibility.

To resolve these structural deadlocks, we propose a new framework--\textbf{Monotone Alternating Splines (MAS)}--to model CCIF. 
MAS replaces the rigid MNN with piecewise monotone splines, which offer greater generality, flexibility, and efficiency, and enjoy strong theoretical guarantees for both fitting and generalization. 
Specifically, MAS separates the estimated CCIF into two distinct sub-components:

Firstly, the \textbf{MAS interpolation} explicitly divides the CCIF into multiple fine-grained intervals and leverages monotone splines to accurately fit each segment.
The parameters of each interpolation segment are determined by a sequence encoder (e.g., a Transformer), and then the monotone splines are connected with $C^1$ continuity.
Through this fine-grained segmented fitting, MAS can accurately fit complex TPP sequences. Theoretically, it possesses an approximation capability comparable to existing parameterization methods, alleviating the MNN convexity restriction.

Secondly, because standard piecewise monotone splines cannot maintain global monotonicity, the \textbf{MAS extrapolation} component is introduced. 
Specifically, outside the interpolation intervals, MAS employs a simple monotonically increasing function (e.g., a linear or linear-plus-exponential function) that smoothly connects with the tail of the internal interpolation part. The extrapolation guarantees global monotonicity over the unbounded time horizon. Furthermore, it allows MAS to recover several classical TPP baselines, thereby securing robust generalization for distant future predictions.

Beyond the architecture, we establish a comprehensive theoretical foundation for the MAS framework. First, we analyze the generalization error of MAS, explicitly decomposing it into the interpolation error, the extrapolation error, and the complexity error. Based on this established generalization bound, we theoretically validate how the extrapolation component enhances the generalization capability of MAS. 
Furthermore, by comparing the approximation error bounds between MAS and MNNs, we theoretically prove that MAS possesses superior fitting and generalization capabilities compared to MNNs. 
Finally, we provide theoretical guidance on how to configure optimal MAS hyperparameters.

Our contribution can be summarized as follows:
\begin{itemize}
    \item We expose the fundamental limitations of existing MNNs for CCIF modeling. We demonstrate that MNN architecture inherently traps the model in three structural deadlocks, critically restricting model expressivity for complex TPP dynamics.
    \item We propose Monotone Alternating Splines. By introducing a highly expressive interpolation component and a globally stable extrapolation component, MAS satisfies all mathematical requirements for CCIF modeling, delivering higher flexibility and computational efficiency.
    \item We establish a comprehensive theoretical foundation for MAS. By decomposing and bounding MAS's generalization error, we prove its significant superiority over MNNs. 
    Moreover, we explicitly provide a theoretical guideline for MAS hyperparameter selection. Extensive experiments across diverse synthetic and real-world event datasets validate the effectiveness of MAS.
\end{itemize}

The remainder of this paper is organized as follows. \cref{related_work} reviews related work. \cref{preliminary} formulates the preliminaries of CCIF and its required conditions.
Next, we analyze the limitations of MNNs in CCIF modeling in \cref{analysis}.
\cref{method} details the mathematical formulation of the MAS framework. In \cref{theory}, we detail the theoretical analyses comparing structural approximation bounds and formalizing generalization guidelines. \cref{experiments} presents the experimental evaluations. Finally, \cref{conclusion} concludes the paper.

\section{Related Work} 
\label{related_work}

This section reviews related work on TPPs and monotone splines. 

\subsection{Temporal Point Processes}
Temporal Point Processes provide the mathematical framework for modeling asynchronous event sequences in continuous time \citep{daley2007introduction}. Classical statistical methods typically parameterize the CIF with predefined functional forms under specific assumptions. For example, the homogeneous Poisson process models independent event occurrences \citep{kingman1992poisson}; the Hawkes process captures self-exciting phenomena, where past events increase the probability of future arrivals \citep{hawkes1971spectra,bacry2015hawkes}; and the Self-correcting process models stress accumulation that decreases after an event \citep{isham1979self}. 
However, these classical models, whether optimized via Maximum Likelihood Estimation (MLE) or Bayesian frameworks, still struggle to capture the highly nonlinear, multimodal temporal dependencies common in modern large-scale datasets.

\subsection{CIF-based Neural TPPs} 
To improve structural expressiveness, deep learning has shifted TPP modeling toward data-driven neural methods. Early approaches used recurrent architectures to parameterize the exact intensity function. Specifically, RMTPP \citep{du2016recurrent} and NHP \citep{mei2017neural} map historical event embeddings into a dynamic continuous-time CIF. Later methods leveraged self-attention to capture long-range interactions, leading to autoregressive models such as THP \citep{simiao2020transformer} and SAHP \citep{zhang2020SAHP}. In parallel, continuous-time state formulations based on Ordinary Differential Equations \cite{chen2018neural,rubanova2019latent} and Stochastic Differential Equations \citep{2024_ICML_NJDTPP} were proposed to flexibly model intensity trajectories between events. Despite improved predictive performance, these CIF-based models share a key computational bottleneck: exact Negative Log-Likelihood (NLL) evaluation requires integrating the CIF over continuous time. Since neural CIFs generally lack closed-form integrals, their integration must rely on Monte Carlo sampling \citep{hastings1970monte} or numerical quadrature \citep{golub1969calculation}, which introduces approximation error and slows training.

\subsection{Alternative TPP Parameterizations and CCIF Modeling} 
To bypass the CIF bottleneck, recent work has explored alternative parameterizations and estimation strategies. Some methods avoid likelihood integration entirely by using  Normalizing Flows \citep{shchur2020fast}, or Score Matching frameworks \citep{li2023smurf, cao2025scorematching}. Meanwhile, other approaches model intermediate variables equivalent to the CIF to avoid integration. For example, some methods model the conditional distribution of the next event time using mixture distributions \citep{shchur2019intensity, panosdecomposable}, while others directly model the CCIF, which is the integral of the CIF, and then obtain the exact CIF via automatic differentiation. Existing CCIF-based models \citep{omi2019fully, liu2024cumulative,wang2024cumulative} typically adopt MNNs \citep{Sill1997MNN} to enforce the positive-intensity prerequisite. Specifically, our method focuses on improving current CCIF approaches.

\subsection{Monotone Splines in Continuous Modeling}
Monotone splines have a long history in statistical numerical analysis as tools for modeling monotone data \citep{ramsay1988monotone}. Early work established methods for constructing smooth, monotonically increasing trajectories using piecewise basis functions \cite{CubicSpline1980Fritsch}. Later methods introduced Rational Linear Splines (RLS) \citep{2020RLS}, Rational Quadratic Splines (RQS) \cite{RQS1987Geogeory}, and Rational Cubic Splines (RCS) \cite{RCS2012}. These rational monotone splines guarantee strict monotonicity and high-order smoothness, and they also support stable analytical inverses. Moreover, each segmented interval can be controlled independently by localized parameters without affecting global monotonicity, which provides high flexibility \citep{durkan2019neuralflows}.
In this work, we identify the structural deadlocks of previous CCIF modeling methods and leverage monotone splines to introduce MAS as a more expressive alternative with stronger approximation guarantees.

\section{Preliminaries} \label{preliminary}

In this section, we provide the mathematical background on TPPs, CCIF, and how to structurally model the CCIF sequence following certain events. 

\subsection{Temporal Point Processes} \label{tpp}
A TPP is a probabilistic model that describes sequences of events occurring within a continuous time window $[0, T]$. A realization of a multivariate TPP is represented as an event sequence $S = \{(t_n, k_n)\}_{n=1}^{N}$, where $N$ is a random variable denoting the total number of events. Here, $0 < t_1 < \dots < t_N < T$ are the event timestamps in continuous time, and $k_n \in \{1,\dots,K\}$ is the mark (or event type) associated with the $n$-th event. 

TPPs are primarily characterized by the CIF $\lambda_k^*(t \mid \mathcal{H}_{t^-})$, which defines the instantaneous occurrence rate of a type-$k$ event given the historical data $\mathcal{H}_{t^-} = \{(t_n, k_n) : t_n < t\}$:
$$
    \lambda_k^*(t) = \lim_{\delta_t \to 0} \frac{p(\text{type-}k\text{ event} \in [t,t+\delta_t] \mid \mathcal{H}_{t^-})}{\delta_t}, 
$$
where the asterisk indicates conditioning on the history up to but not including time $t$. 
Given the CIF, the log-likelihood of a multivariate TPP for an observed sequence $S$ is:
\begin{equation}
    \log p(S) = \sum_{n=1}^{N} \log\lambda_{k_n}^*(t_n) -\int_0^T \lambda^*(t) dt, 
\label{likelihood}
\end{equation}
where $\lambda^*(t)=\sum_{k=1}^K\lambda^*_k(t)$ is the total intensity across all event types. Model parameters are estimated by maximizing the log-likelihood in \cref{likelihood}. Additionally, given the history $\mathcal{H}_{t_n}$, the probability density function of the next event timestamp $t_{n+1}$ is:
$$
    p(t_{n+1} \mid \mathcal{H}_{t_n}) = \lambda^*(t_{n+1}) \exp\left(-\int_{t_{n}}^{t_{n+1}} \lambda^*(t)dt\right). 
$$
We can estimate the expected next event timestamp $\widehat{t}_{n+1}$ and the most likely mark $\widehat{k}_{n+1}$ using the following estimators:
\begin{equation}
\widehat{t}_{n+1}=\int_{t_n}^\infty tp(t \mid \mathcal{H}_{t_n})dt, \ \ \widehat{k}_{n+1}=\arg\max_k\frac{\lambda_{k}^*(\widehat{t}_{n+1})}{\lambda^*(\widehat{t}_{n+1})}.
\label{predicting t_n+1}
\end{equation}

\subsection{CCIF Parameterization}
A major drawback of standard CIF-based modeling is that the integral $\int_0^T \lambda^*(t) dt$ in \cref{likelihood} generally lacks a closed-form solution when parameterized by complex neural networks, thereby requiring computationally inefficient numerical integration. An alternative, mathematically equivalent formulation is to model the CCIF directly: $\Lambda_k^*(t) = \int_0^t \lambda_k^*(\tau) d\tau$. Under this framework, \cref{likelihood} can be rewritten as:
\begin{equation}
    \log p(S) = \sum_{n=1}^{N} \log\frac{d}{dt}\Lambda_{k_n}^*(t_n^-) - \Lambda^*(T), 
\label{CCIF-likelihood}
\end{equation}
where $\Lambda^*(t) = \sum_{k=1}^K \Lambda_{k}^*(t)$, and $\frac{d}{dt}\Lambda_{k_n}^*(t_n^-)$ represents the left-hand derivative of $\Lambda_k^*(t)$ evaluated at $(t_n, k_n)$. Incorporating the CCIF allows MLE to bypass integral computations entirely. The required instantaneous intensity is easily and precisely recovered via automatic differentiation.

\subsection{Modeling CCIF after Certain Events}
Following each observed event $t_1,t_2,\cdots,t_N$, the continuous global function $\Lambda^*(t)$ can be decomposed and be expressed as:
\begin{equation}
    \Lambda^*(t) = \sum_{n=0}^{N} {I}_{(t_n,t_{n+1}]}\Lambda^{*(n)}(t), 
\label{eq: decomposes}
\end{equation}
where $t_0=0$, $t_{N+1}=T$, and ${I}_{(t_n,t_{n+1}]}$ is an indicator function taking the value 1 if $t \in (t_n, t_{n+1}]$ and 0 otherwise. Here, $\Lambda^{*(n)}(t) = \int_{t_n}^t \lambda^*(\tau) d\tau + \Lambda^{*(n-1)}(t_n)$ represents the CCIF following the $n$-th event, with the base case $\Lambda^{*(-1)}(0) = 0$. We emphasize that the natural support domain of $\Lambda^{*(n)}(t)$ is $(t_n, \infty)$ because, during inference, the subsequent event $t_{n+1}$ is unknown and could theoretically occur indefinitely far in the future. During training, $t_{n+1}$ is observed, allowing us to cleanly truncate $\Lambda^{*(n)}(t)$ to $(t_n, t_{n+1}]$ via the indicator function, subsequently concatenating $\Lambda^{*(n)}(t)$ and $\Lambda^{*(n+1)}(t)$ to reconstruct the exact global CCIF $\Lambda^*(t)$.

\section{Failures of Monotone Neural Networks}
\label{analysis}

\begin{figure}[t]
\centering
    \includegraphics[width=\linewidth]{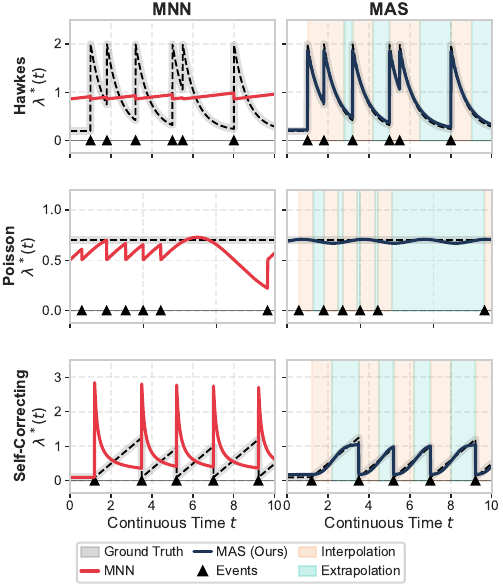}
        \caption{Performance comparison of MNNs and MAS on three classical TPPs (Hawkes, Poisson, Self-Correcting). The left column shows the three MNN deadlocks—Convexity Restriction, Saturation Limit, and Violations of CCIF Requirements—while the right column shows that MAS overcomes these limitations and accurately reconstructs the ground-truth dynamics.}
    \label{fig:demo_trilemma}
\end{figure}

In this section, we investigate the principles of using MNNs for CCIF modeling. We show that positive-weighted MLPs restrict MNN expressiveness and prevent them from modeling complex TPP dynamics. 

\subsection{MNNs for TPP Modeling}
\label{sec:mnn_principles}

We first explain the MNN-based parameterization of CCIF.
Specifically, an $L$-layer MNN $\Lambda^*(t)$ can be expressed as:
$$
\begin{aligned}
    &\boldsymbol{h}^{(1)} = \sigma\!\left(\boldsymbol{W}^{(1)} t + \boldsymbol{b}^{(1)}\right), \\
    &\boldsymbol{h}^{(l)} = \sigma\!\left(\boldsymbol{W}^{(l)} \boldsymbol{h}^{(l-1)} + \boldsymbol{b}^{(l)}\right), \quad l=2,\dots,L-1, \\
    &\Lambda^*(t) = \boldsymbol{W}^{(L)} \boldsymbol{h}^{(L-1)} + \boldsymbol{b}^{(L)},
\end{aligned}
$$
where $t$ is the elapsed time, $\sigma(\cdot)$ is the element-wise activation function, $\boldsymbol{W}^{(l)}$ and $\boldsymbol{b}^{(l)}$ are the connection weights and biases, respectively. The corresponding CIF is derived as:
\begin{equation}
\label{equ: MNN CIF}
        \lambda^*(t)
    = \frac{\partial \Lambda^*(t)}{\partial t}
    = \boldsymbol{W}^{(L)}
    \left(
    \prod_{l=L-1}^{1}
    \mathrm{diag}\!\left[\sigma'(\boldsymbol{z}^{(l)})\right]
    \boldsymbol{W}^{(l)}
    \right).
\end{equation}

The core motivation behind using MNNs for modeling CCIF is grounded in the following theorem. 
\begin{theorem}
\label{MNN Monotonicity}
If all the model weights are strictly constrained to be non-negative ($\boldsymbol{W}^{(l)} \ge 0, \forall l$) and the activation function is non-decreasing ($\sigma'(x) \ge 0, \forall x$), then the resulting derivative is guaranteed to be non-negative: 
$\lambda^*(t)\ge 0.$
\end{theorem}
While \cref{MNN Monotonicity} ensures MNN's monotonicity, it simultaneously traps the model in three structural deadlocks, as discussed in the following subsection.

\subsection{Three Structural Deadlocks of MNNs}
\label{sec:mnn_limitations}

To construct a practical MNN, designers must make choices regarding the non-decreasing activation function $\sigma(\cdot)$ in \cref{MNN Monotonicity}. However, we demonstrate that these choices force the MNN into one of three structural deadlocks: the Convexity Restriction, the Saturation Limit, or the Violations of CCIF Requirements.

\subsubsection{Convexity Restriction}

First, when equipped with standard convex activation functions, the resulting MNN becomes strictly convex, preventing the model from capturing any decaying intensity patterns. 
While unbounded and convex activations like ReLU or Softplus are favored by modern deep learning paradigms, these common activation functions conflict with the MNN architecture in \cref{MNN Monotonicity}. 
\begin{theorem}
\label{thm: convexity trap}
Let the CCIF be parameterized by an MNN mapping $t \to \Lambda^*(t)$. If the activation function is convex ($\sigma''(x) \ge 0$) and the weights are non-negative ($\boldsymbol{W}^{(l)} \ge 0$), then the CCIF is convex globally:
$
 \Lambda^{*\prime\prime}(t) = \lambda^{*\prime}(t) \ge 0, \quad \forall t.
$
\end{theorem}

\cref{thm: convexity trap} reveals a severe side-effect: positive linear combinations of convex functions preserve $\Lambda^*(t)$'s convexity strictly.
If common convex activation functions are applied in MNNs, the magnitude of the CIF $\lambda^*(t)$ can only monotonically increase or remain completely flat. 
This severely limits the model's expressive power. 
In real-world dynamics, such as an exponentially decaying Hawkes process, the CIF spikes after an event occurrence and then smoothly decays (i.e., $\lambda^{*\prime}(t) < 0$). An MNN bounded by \cref{thm: convexity trap} inherently fails to fit such a decreasing CIF, regardless of its depth or width.

\subsubsection{Saturation Limit}


Second, to bypass the aforementioned convexity restriction, an intuitive approach is to replace ReLU with S-shaped activations (e.g., Tanh or Sigmoid), which offer concave regions ($\sigma'' < 0$), thus permitting the derivative $\lambda^*(t)$ to decline. Unfortunately, this exposes the MNN to the following saturation limit. 
\begin{theorem}
\label{thm: saturation limit}
Let $\sigma(\cdot)$ be bounded such that $\sigma(x) \in [c_{\min}, c_{\max}]$. Assuming finite network weights ($\boldsymbol{W}^{(l)} < \infty$), there exists a rigorous finite upper bound $M < \infty$ for the generated CCIF:
$
\lim_{t \to \infty} \Lambda^*(t) \le M < \infty.
$
\end{theorem}

As is implied in \cref{thm: saturation limit}, non-convex activation functions force the CCIF $\Lambda^*(t)$ into an enforced upper limit $M$, precluding the mechanism from modeling persistent event streams. 
For any ordinary sequence demonstrating a baseline rate $\mu > 0$, the mathematically correct total expected events must approach infinity over time ($\lim_{t \to \infty} \Lambda^*(t) = \infty$). 
Consequently, Tanh- or Sigmoid-based MNNs fundamentally fail to represent a stationary long-term trend; the learned intensity $\lambda^*(t)$ will invariably exhaust itself and force the predicted sequence to permanently freeze. 

\subsubsection{Violations of CCIF Requirements}
Since using either convex or bounded, non-convex activation functions leads to the shortcomings of MNN-based $\Lambda^*(t)$, the remaining option is to apply a non-convex transformation $t\rightarrow F(t)$ to time 
$t$ in \cref{equ: MNN CIF} in advance, yielding:
$$
\Lambda^*(t) = \mathrm{MNN}(F (t)),
$$
where $\mathrm{MNN}(\cdot)$ is a positive weighted MLP.
For example, FullyNN \citep{omi2019fully} leverages a log-transformation $F(t) = \log(t)$. Because $\log t$ is intrinsically concave, this pre-processing step avoids the Convexity Restriction without requiring the $\mathrm{MNN}(\cdot)$ itself to possess non-convex activations. 

Nevertheless, this mathematical tactic still introduces inevitable structural flaws. It violates the boundary condition of CCIF and restricts the expressive power of the model:
\begin{itemize}
    \item \textbf{Violation of the Initial Physical Condition.} Pre-transforming $t$ permanently introduces an unavoidable intercept at $t=0$, causing inaccurate NLL optimization. 
    By definition, the cumulative integration at exact time zero must be zero: $\Lambda^*(0) = \int_0^0 \lambda^*(\tau) d\tau = 0$. However, evaluating the deformed CCIF sequence at $t \to 0^+$ yields a limit breakdown:
    $$
    \Lambda^*(0) = \lim_{t \to 0^{+}} \mathrm{MNN}(\log t) = \mathrm{MNN}(-\infty)\neq 0.
    $$
    Because MNN is a nested composition of affine layers and non-negative activations, its value at $-\infty$ is unpredictable and can only be approximated during initialization. This inevitably leaves a persistent positive intercept $\Lambda^*(0) = C \neq 0$. Such a non-zero starting point directly interferes with the exactness of the likelihood computation in \cref{eq: decomposes}.
    
    \item \textbf{Enforced Heuristic Decay Constraints.} Additionally, deforming time $t$ restricts the intrinsic trend of the derived CIF, severely degrading the network's capability to model diverse TPP behaviors. 
    According to the chain rule, CIF $\lambda^*(t)$ is tied to the derivative of $\log(t)$:
    $$
    \lambda^*(t) = \frac{\partial \mathrm{MNN}(\log t)}{\partial (\log t)} \cdot \frac{1}{t}=O(t^{-1}).
    $$
    This mathematical entanglement dictates an $O(t^{-1})$ deterministic decay pattern acting as a global multiplier. While this $1/t$ allows the MNN to roughly approximate TPPs with decaying CIFs, it contradicts other basic TPP behaviors. For instance, in a Self-Correcting process, the CIF accumulates and increases over time between events. 
    The inherent $1/t$ penalty suppresses this upward trend, making the model poorly suited for ascending dynamics. 
\end{itemize}

Consequently, MNN-based CCIF methods face three fundamental deadlocks. As shown in \cref{fig:demo_trilemma}, continuous-time MNNs inevitably suffer from expressive limitations and boundary violations, regardless of activation design or time deformation. To overcome this bottleneck, MAS replaces MNNs with piecewise monotone splines for explicit CCIF parameterization, yielding a more general and flexible framework. 

\section{Methodology} \label{method}
In this section, we introduce a novel CCIF modeling method: MAS, and detail its implementation.
For complex TPPs, the exact CCIF after event $t_n$, $\Lambda^{*(n)}(t)$, can be highly intricate across the interval $[t_n, \infty)$. Directly fitting this complete trajectory using a single elementary function or a rigid MNN is generally suboptimal. 

Consequently, a natural idea is to divide the entire CCIF $\Lambda^{*(n)}(t)$ into smaller segments that are easier to fit. Inspired by this idea, we propose a flexible parameterization framework termed \textbf{MAS} to model the CCIF. MAS consists of two parts: an \textbf{interpolation} component to fit the recent CCIF dynamics right after an event, and an \textbf{extrapolation} component to model the long-term CCIF tail. 
Formally, MAS models the CCIF after the $n$-th event as: 
\begin{equation}
\label{eq:MAS_for_CCIF}
\Lambda_\theta^{*(n)}(t)=
\underbrace{\sum_{m=1}^M {I}_{(w^{(n)}_{m-1},w^{(n)}_m]}(t) f^{*(n)}_{m,\theta}(t)}_{\text{MAS interpolation}}
+
\underbrace{{I}_{(w^{(n)}_{M},\infty)}(t) g^{*(n)}_\theta(t)}_{\text{MAS extrapolation}},
\end{equation}
where the interval $(t_n, \infty)$ is partitioned into $M+1$ contiguous segments. The first $M$ segments form the interior $(w^{(n)}_0, w^{(n)}_1], \dots, (w^{(n)}_{M-1}, w^{(n)}_M]$, with the first knot starting exactly at the event time $w^{(n)}_0 = t_n$. The $m$-th piecewise function $f^{*(n)}_{m,\theta}(t)$ models the MAS interpolation within its respective interval, while the function $g^{*(n)}_\theta(t)$ represents the MAS extrapolation over $(w^{(n)}_M, \infty)$. The vector $\theta$ encapsulates the underlying neural network parameters driving the model. 

\subsection{MAS Interpolation}
\label{sec:mas_interpolation}

The MAS interpolation directly determines the near-term variation trend of the CCIF. By segmenting $\Lambda_\theta^{*(n)}(t)$,  the MAS interpolation provides high flexibility. However, because $\Lambda_\theta^{*(n)}(t)$ is the integral of a positive CIF, a valid MAS interpolation must satisfy the following requirements:

\textbf{(1) Global Monotonicity:} Since the intensity $\lambda^*(t) \ge 0$, $\Lambda_\theta^{*(n)}(t)$ must be non-decreasing. For any $t_n < t_1 \le t_2 < \infty$, it must hold that $\Lambda_\theta^{*(n)}(t_1) \le \Lambda_\theta^{*(n)}(t_2)$.

\textbf{(2) $C^1$-Continuity Between Timestamps:} For any continuous interval strictly between two events $t_n < t < t_{n+1}$, $\Lambda^*(t)$ must be continuous and possess a continuous derivative $\Lambda^{*\prime}(t)$. 

\textbf{(3) $C^0$-Continuity At Timestamps:} At each timestamp $t_n$, the CCIF is continuous ($\Lambda^{*}(t_n^-) = \Lambda^{*}(t_n^+)$), but its derivative is permitted to be discontinuous ($\Lambda^{*\prime}(t_n^-) \neq \Lambda^{*\prime}(t_n^+)$). 

The second and third conditions ensure that the CIF remains continuous between any two consecutive timestamps but may exhibit jumps at the timestamps, which is a core characteristic of many history-dependent TPPs, such as the Hawkes process and the Self-Correcting process.
To ensure the generated $\Lambda_\theta^{*(n)}(t)$ perfectly satisfies monotonicity and $C^1$-continuity within $(t_n, t_{n+1})$, the adjacent functions in MAS must be seamlessly connected at the partition knots (i.e., $w^{(n)}_1, \dots, w^{(n)}_M$) in both function values and first derivatives:
\begin{equation}
\label{eq:conditions_for_MAS}
\resizebox{0.9\linewidth}{!}{$
\begin{aligned}
    f^{*(n)}_{m,\theta}(w^{(n)}_{m}) = f^{*(n)}_{m+1,\theta}(w^{(n)}_{m}),& 
    f^{*(n)\prime}_{m,\theta}(w^{(n)}_{m}) = f^{*(n)\prime}_{m+1,\theta}(w^{(n)}_{m}), \\
    f^{*(n)}_{M,\theta}(w^{(n)}_{M}) = g_{\theta}^{*(n)}(w^{(n)}_{M}),& 
    f^{*(n)\prime}_{M,\theta}(w^{(n)}_{M}) = g_{\theta}^{*(n)\prime}(w^{(n)}_{M}).
\end{aligned}$}
\end{equation}

Piecewise monotone splines perfectly satisfy these requirements, because their analytical forms are directly parameterized by the endpoint values and derivatives. Taking the Rational Quadratic Spline (RQS) as an example, given an interval $[w_{m-1}, w_{m}]$, boundary values $\{y_{m-1},y_m\}$, and boundary derivatives $\{\delta_{m-1},\delta_{m}\}$, the $m$-th piecewise RQS is:
\begin{equation}
\label{eq:RQS_example}
f_{m}(t) = y_{m-1} + \frac{\left(y_{m} - y_{m-1}\right) \left[s_m \tau^2 + \delta_{m-1} \tau (1 - \tau)\right]}{s_{m} + \left(\delta_{m} + \delta_{m-1} - 2s_{m}\right) \tau (1 - \tau)},
\end{equation}
where $s_m=\frac{y_{m}-y_{m-1}}{w_m-w_{m-1}}$ and $\tau=\frac{t-w_{m-1}}{w_m-w_{m-1}} \in [0, 1]$. By substituting $\tau = 0$ (i.e., $t = w_{m-1}$) and $\tau = 1$ (i.e., $t = w_{m}$) into Eq. \ref{eq:RQS_example} and its corresponding derivative, we obtain:
\begin{equation}
\label{eq:RQS_property}
\begin{aligned}
    f_{m}(w_{m-1}) &= y_{m-1}, \quad &f_{m}(w_{m}) &= y_{m}, \\
    f'_{m}(w_{m-1}) &= \delta_{m-1}, \quad &f'_{m}(w_{m}) &= \delta_{m}.
\end{aligned}
\end{equation}
This implies that an RQS spline piece is uniquely determined by its boundary constants $(y_{m-1}, y_m)$ and derivatives $(\delta_{m-1}, \delta_{m})$. Therefore, as long as adjacent splines share the same structural parameters $\{(w_m, y_m, \delta_m)\}_{m=0}^M$, they are connected with exact $C^1$-continuity. Consequently, all existing monotone splines can act as the basis for MAS interpolation. In Appendix B, we discuss several specific variations and their implementation.

\begin{figure}[t]
\includegraphics[clip,width=0.9\linewidth]{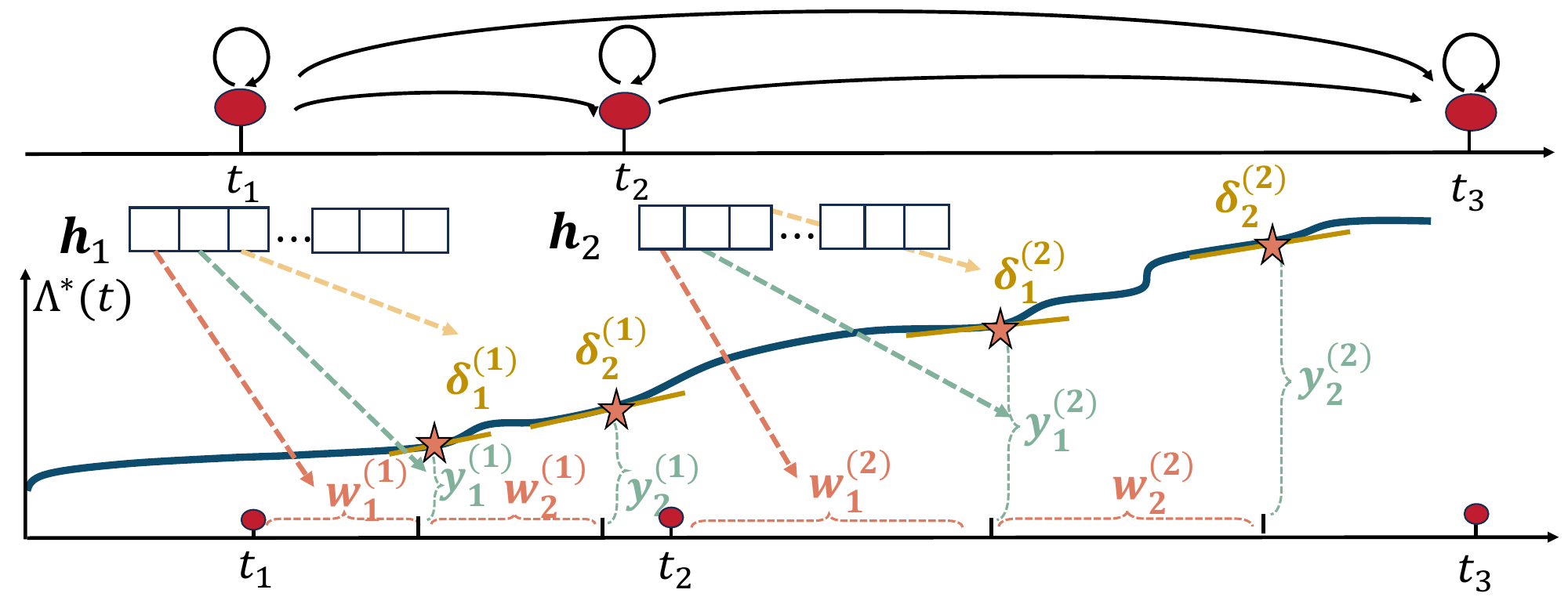}
    \caption{Implementation of MAS interpolation. The history encoder maps past events into a hidden representation, from which an MLP predicts positive interval widths, cumulative increments, and knot derivatives; cumulative summation then constructs an MAS interpolation.}
    \label{fig: Parametrization of Monotone Splines}
\end{figure}

\textbf{Parameterization of MAS Interpolation}:
Next, we discuss how the MAS interpolation is implemented. 
In practice, the defining parameters of MAS—the interpolation knots $\{w^{(n)}_m\}_{m=1}^M$, cumulative function values $\{y^{(n)}_m\}_{m=1}^M$, and derivatives $\{\delta^{(n)}_m\}_{m=1}^M$—are all unknown and must be derived from historical data using autoregressive or recurrent encoders.
Specifically, at timestamp $t_n$, we employ an encoder (e.g., a Transformer \citep{simiao2020transformer}, parameterized by $\theta_1$) to compress the event history $\mathcal{H}_{t_n}$ into an embedding $\mathbf{h}_n$. Then, an MLP (parameterized by $\theta_2$) followed by a Softplus activation function maps $\mathbf{h}_n$ to strictly positive step intervals ($\Delta w_m, \Delta y_m, \delta_m$). The knots are obtained cumulatively:
$$
\begin{aligned}
    w^{(n)}_m = w^{(n)}_{m-1} + \Delta w_m,\ \  y^{(n)}_m = y^{(n)}_{m-1} + \Delta y_m.
\end{aligned}
$$
To ensure the seamless concatenation of consecutive CCIF sequences at event boundaries (the $C^0$-continuity requirement), we set $w^{(n)}_0 = t_n$ and $y^{(n)}_0 = \Lambda^{*(n-1)}(t_n)$. For adjacent interpolation segments, the network only predicts a single common derivative $\delta^{(n)}_{m}$ at knot $w^{(n)}_m$, natively forcing $f^{*(n)}_{m,\theta}$ and $f^{*(n)}_{m+1,\theta}$ to share the identical gradient.

\textbf{Fitting Ability of MAS Interpolation}: 
\textcolor{black}{
By introducing piecewise monotone splines, the fitting capability of MAS is significantly enhanced. Unlike MNNs, the theoretically superior expressiveness of MAS can be rigorously supported. Specifically, shorter interpolation intervals allow the MAS interpolation to approximate the ground-truth $\Lambda^{*(n)}(t)$ more accurately. }
Formally, we obtain the following theorem. 
\begin{theorem}
\label{thm: accuracy of monotone splines}
Suppose the MAS interpolation support can extend to infinity, i.e., \( w^{(n)}_{M} \to \infty \). Then, for any ground-truth \( \Lambda^*(t) \), there exists an MAS representation $\Lambda_\theta^*(t) = \sum_{n =0}^N \mathcal{I}_{(t_n, t_{n+1}]} \Lambda_\theta^{*(n)}(t)$, such that: 
\begin{equation}
    \Vert\Lambda^*(t)-\Lambda_\theta^*(t)\Vert_0\le \frac{1}{2}C_0\Delta^2, 
\end{equation}
where \( \Vert f(t) \Vert_0 = \underset{t>0}{\sup} \vert f(t) \vert \), \( \Delta \) is the maximum interpolation interval, and \( C_0 \) is a constant independent of the data distribution. 
Furthermore, there exists an $\Lambda_\theta^*(t)$, such that, 
\begin{equation}
    \Vert\Lambda^{*\prime}(t)-\Lambda_\theta^{*\prime}(t)\Vert_0 \le C_0\Delta,
\end{equation}
where, at the timestamp $t_n$, the above derivatives are taken as the left-hand derivatives. 
\end{theorem}
\cref{thm: accuracy of monotone splines} states that 
MAS fits any ground-truth CIF and CCIF provided that the interpolation intervals are sufficiently small, which implies its fitting performance is comparable to any other TPP parameterization.

\subsection{MAS Extrapolation}
\label{sec:mas_extrapolation}
\textcolor{black}{Next, we discuss another necessary component, MAS extrapolation, and how it is implemented.}
\cref{thm: accuracy of monotone splines} assumes that the MAS interpolation support can extend to infinity, i.e., \( w^{(n)}_{M} \to \infty \). However, in practice, this is not possible. Therefore, an additional monotonic extrapolation function \( g^{*(n)}_\theta(t) \) must be used to model $\Lambda^{*(n)}(t)$ over \( (w^{(n)}_{M},\infty) \). 
It is worth noting that most monotone splines (without any additional tail design) are based on polynomials, which do not guarantee monotonicity outside the interpolation interval. Taking the RQS in \cref{eq:RQS_example} as an example, when \( s^{(n)}_m \) is smaller than \( \delta^{(n)}_{m-1} \) and \( t \) is sufficiently large, \( f^{(n)}_{m}(t) \) becomes monotonically decreasing and can even take negative values. See \cref{fig:differences in MAS and monotone splines} for an illustration.  
Thus, an alternative parameterization is required for \( g^{*(n)}_\theta(t) \) in \cref{eq:MAS_for_CCIF} to ensure the global monotonicity of $\Lambda^{*(n)}(t)$ over \( (w^{(n)}_{M},\infty) \). 

\begin{figure}[t]
\includegraphics[clip,width=0.9\linewidth]{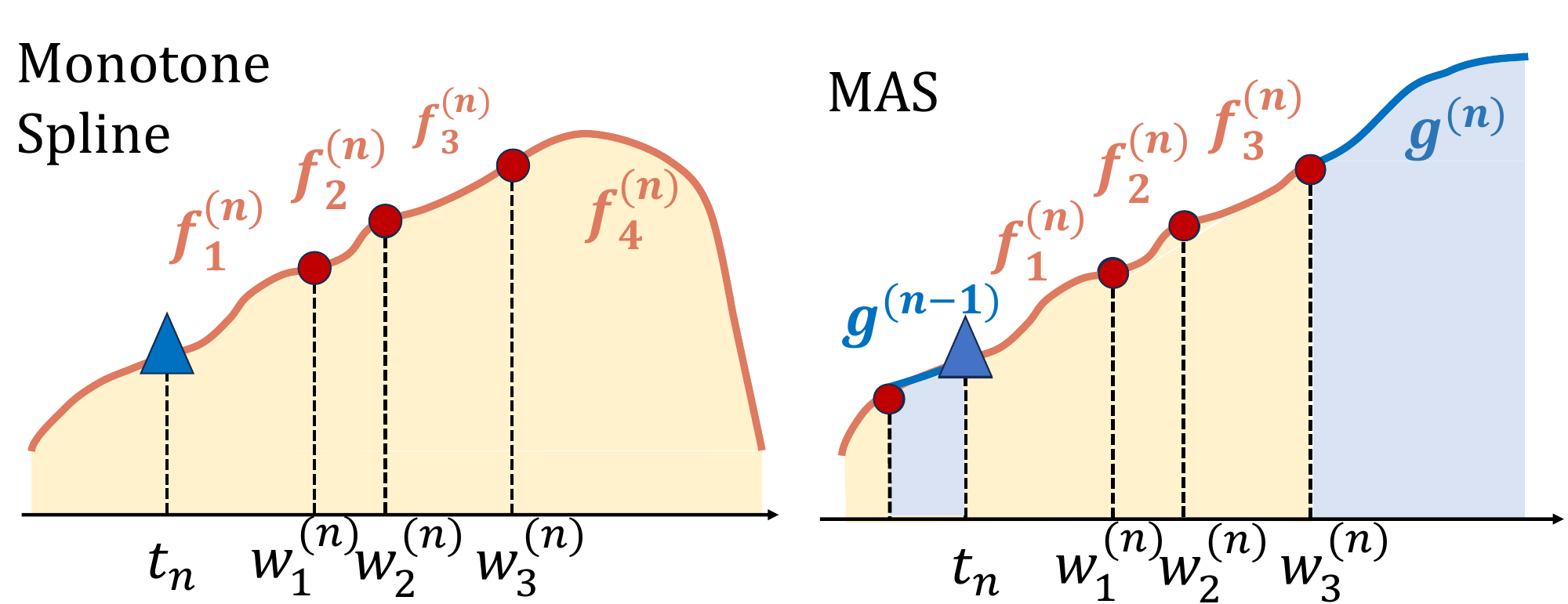}
\caption{Comparison between MAS and a standalone monotone spline. A monotone spline may fail to preserve monotonicity outside its interpolation support, whereas MAS attaches a monotone extrapolation function after the final knot to maintain global monotonicity and smooth connection of the CCIF.}
    \label{fig:differences in MAS and monotone splines}
\end{figure}

In implementation, \( g^{*(n)}_\theta(t) \) can be a linear function \( g^{*(n)}_\theta(t) = a t + b \) with \( a,b > 0 \), or a linear-plus-exponential function \( g^{*(n)}_\theta(t) = c + b t - \exp(-a t) \) with \( a,b,c > 0 \). 
Similar to the MAS interpolation, the parameters $a$, $b$, and $c$ are predicted from the history embedding $\mathbf{h}_n$ by the MLP. Notably, when connecting the interpolation and the extrapolation at the final knot $w^{(n)}_M$, we must guarantee \cref{eq:conditions_for_MAS}. To ensure this smoothness easily, we first determine the extrapolation function $g^{*(n)}_\theta(t)$ from the network. Then, we strictly assign the final knot target as $y^{(n)}_M = g^{*(n)}_{\theta}(w^{(n)}_M)$ and the final derivative target as $\delta^{(n)}_M = g^{*(n)\prime}_{\theta}(w^{(n)}_M)$. According to the properties of monotone splines in \cref{eq:RQS_property}, substituting these targets shapes the final interpolation piece $f_M$, satisfying the strict $C^1$-continuity of $\Lambda^{*(n)}(t)$.

\textbf{Equivalence to Classical TPPs}: 
As an important hyperparameter in \cref{eq:MAS_for_CCIF}, 
a larger \( w^{(n)}_{M} \) increases the contribution of piecewise monotone spline interpolation, enhancing the model's fitting performance. Conversely, a smaller \( w^{(n)}_{M} \) amplifies the influence of the extrapolation component, improving the model's generalization performance (further discussed in \cref{theory}). 
Notably, when \( w^{(n)}_{M} = t_n \), MAS reduces to classical TPP models. 
\begin{corollary}
    A homogeneous Poisson process with intensity \( \mu \) can be expressed by an MAS with linear extrapolation, where \( w^{(n)}_{M} = t_n \) and \( g^{*(n)}(t) = \mu t, t\in[t_n,\infty)\).
\end{corollary}
\begin{corollary}
    An exponential-kernel Hawkes process with the intensity function \( \lambda^*(t) = \mu + \sum_{t_n < t} \alpha \exp(-\beta(t - t_n)) \) can be expressed by an MAS with the sum of linear and exponential extrapolations, where \( w^{(n)}_{M} = t_n \) and 
     \( g^{*(n)}(t) = \mu t + \frac{\alpha}{\beta} \sum_{i=1}^n (1 - \exp(-\beta(t - t_i))), t\in[t_n,\infty) \). 
\end{corollary}
Similar to the two corollaries above, other classical TPPs can also be represented within the MAS framework. By incorporating the piecewise monotone spline interpolation component, MAS extends classical TPPs and significantly enhances the model's flexibility. 

\subsection{Efficiency of MAS}
MAS also enjoys superior efficiency compared to MNN-based methods.
During the model training process using MLE, MAS does not require automatic differentiation to compute time derivatives at each timestamp in \cref{CCIF-likelihood}, since all monotone splines have analytical derivatives. 
While automatic differentiation is necessary for MNN-based methods to recover the intensity from the CCIF, MAS avoids this derivative-computation step and further improves computational efficiency. In Appendix B, we provide a derivation of the analytical derivative for RQS.

\subsection{MAS for Multivariate TPPs}
MAS can be extended to multivariate TPPs. 
If the multivariate TPPs consist of \( K \) different event types, we can use \( K \) MAS models to represent \( K \) CCIFs. 
Given the event history \( \{(t_1, k_1), \dots, (t_n, k_n)\} \), a Transformer (parameterized by \( \theta_1 \)) is employed to encode the historical information into a history embedding \( \mathbf{h}_n \). Then, \( \mathbf{h}_n \) is mapped by \( K \) different MLPs followed by Softplus (parameterized by \( \theta_{2,1}, \dots, \theta_{2,K} \)) to \( K \) sets of distinct variables for \( K \) MAS models. 
Consequently, we obtain \( \Lambda_1^{*}(t), \dots, \Lambda_K^{*}(t) \). The overall model parameters are given by \( \theta = \{\theta_1, \theta_{2,1}, \dots, \theta_{2,K}\} \). 
Subsequently, we train the model using \cref{CCIF-likelihood} and utilize \cref{predicting t_n+1} to predict the next timestamp and event type. 

\section{Generalization Analysis}
\label{theory}

This section explores the theoretical generalization guarantees of Monotone Alternating Splines (MAS). We systematically outline our analysis in three aspects: first, we establish the generalization bound of MAS; second, we formally prove its theoretical superiority over conventional MNNs in overcoming inherent approximation barriers; and finally, we provide principled guidelines for MAS hyperparameter selection based on the established bounds.

\subsection{Generalization Bound of MAS}

To assess the model's reliability on unseen data, we evaluate the gap between the empirical negative log-likelihood and its expectation. Suppose the training set consists of \( Z \) i.i.d. point process sequences \( S_1, \dots, S_Z \) on the time window $[0,T]$. The \( n \)-th timestamp of the \( z \)-th sequence is denoted as \( t_{zn} \). 
For simplicity, we consider the univariate case, where the max sequence length is $N_0$. Without loss of generality, we assume the interpolation interval is a constant $\Delta>0$. 
The empirical loss is defined as the averaged negative log-likelihood over all sequences: $\mathcal{\widehat{L}}:=-\frac{1}{Z}\sum_{z=1}^{Z}\frac{1}{T}\log p(S_z)$. We train the model by minimizing $\mathcal{\widehat{L}}$ and are interested in its expected value over the data distribution, denoted as $\mathcal{{L}}:=\mathbb{E}[\mathcal{\widehat{L}}]$. 
We can first obtain the following lemma. 

\begin{lemma}
\label{lem: Lipschitz for MAS}
    Given a sequence $S_z = \{t_{z1}, \cdots, t_{zN}\}$ with $N \le N_0$ and an MAS model ${\Lambda}^{*}_\theta(t)$ determined by parameter $\theta$ or $\eta$, the negative log-likelihood loss $-\log p_\theta(S_z)$ is Lipschitz continuous w.r.t. $\theta$: 
    $$
        \left\Vert (-\log p_{\theta}(S_z)) - (-\log p_{\eta}(S_z)) \right\Vert \le \frac{C_1}{T} \sum_{n=1}^N r_n \left\Vert \theta - \eta \right\Vert,
    $$
    where $r_n=m$ if $w_{m-1}^{(n)} \le t_{zn} \le w_{m}^{(n)}$, and $C_1$ is a structural constant independent of the data distribution. 
\end{lemma}

Then, we can derive the error bound between \( \mathcal{L} \) and \( \widehat{\mathcal{L}} \). 

\begin{theorem}
\label{thm: Gen bound for MAS}
Suppose there are $M$ valid interpolation knots within a MAS segment. Then, with probability $1 - \xi$:
\begin{equation*}
\resizebox{\linewidth}{!}{$
\left| \widehat{\mathcal{L}} - \mathbb{E}[\widehat{\mathcal{L}}] \right|
\le \frac{1}{\sqrt{Z}}
\left(
\frac{1}{2}\sqrt{\log \frac{1}{\xi}}
+ C_2 \sqrt{N_0 M} \int_0^{c} \sqrt{C_3 - \log t} \, dt
\right)
$},
\end{equation*}
where $c, C_2, C_3$ are data-independent constants.
\end{theorem}

\cref{thm: Gen bound for MAS} offers the bound between the empirical negative log-likelihood and its expectation, implying minimizing the empirical loss $\widehat{\mathcal{L}}$ helps reduce the expected loss $\mathbb{E}[\widehat{\mathcal{L}}]$. 
Taking a step forward, we combine \cref{thm: accuracy of monotone splines} and \cref{thm: Gen bound for MAS} to derive the upper bound for MAS's generalized error. 

\begin{theorem}
\label{thm: Gen Error for MAS: 2}
Suppose the total interpolation length of MAS is a fixed constant $L$, and each interval of the monotone spline has length $\Delta$. Then, with probability at least $1 - \xi$: 
\begin{equation}
\label{thm:generalization bound} 
\begin{aligned}
    \mathbb{E}[\widehat{\mathcal{L}}]&\le
    {\mathcal{L}^*}+
    \underbrace{\frac{1}{2\sqrt{Z}}\sqrt{{\log\frac1\xi}}}_{\textnormal{probability error}}+\underbrace{\frac{B}{T\cdot L}}_{
        \textnormal{extrapolation error}
    }
     \\&+\underbrace{C_0\left(\frac{\Delta}{C_4}+\frac{\Delta^2}{2}\right)}_{\textnormal{interpolation error}}+
    \underbrace{\frac{C_2}{\sqrt{Z}}\frac{\sqrt{L}}{\sqrt{\Delta}}R(L,\Delta)}_{\textnormal{complexity error}},
\end{aligned} 
\end{equation}
where $\mathcal{L}^*$ is the optimal negative log-likelihood induced by the ground-truth CCIF, $B, C_0, C_2, C_4$ are constants independent of sampling randomness, and $R(L,\Delta)$ is a constant monotonically increasing as $L$ increases or $\Delta$ decreases. 
\end{theorem}

\begin{remark}
    Without loss of generality, the interpolation interval $\Delta$ is set to a constant in our generalization analysis. When implementing MAS, each interpolation knot $w_m^{(n)}$ can be variable and determined by the history information embedding $\mathbf{h}_n$. Even so, \cref{thm: Gen Error for MAS: 2} still holds — one only needs to replace $\Delta$ with the upper bound on the interpolation‐interval length. 
\end{remark}

In Theorem \ref{thm: Gen Error for MAS: 2}, excluding the probability error, the core estimation errors are decomposed into three interpretable parts. First, the \textbf{interpolation error} $C_0\left(\frac{\Delta}{C_4}+\frac{\Delta^2}{2}\right)$ decreases as the interpolation interval \( \Delta \) shrinks, as a smaller \( \Delta \) leads to a more accurate fit to the ground-truth CCIF.
Second, the \textbf{extrapolation error} $\frac{B}{T\cdot L}$ decreases as the total interpolation length \( L \) increases. As the contribution of the interpolation part becomes larger, the influence of the extrapolation part diminishes. 
Finally, the \textbf{complexity error} $\frac{C_2}{\sqrt{Z}}\frac{\sqrt{L}}{\sqrt{\Delta}}R(L,\Delta)$ increases as the number of parameters in MAS increases, implying higher model complexity. 

\subsection{Theoretical Superiority over MNNs}

As highlighted in Section \ref{analysis}, due to the inevitable convexity restriction, convex MNNs fundamentally fail to capture TPPs with decaying CIFs. We formally quantify this limitation as an irreducible error bound.

\begin{theorem}
\label{thm: MAS vs MNN gap}
Let $\Lambda^*(t)$ be the ground-truth CCIF governed by a history-dependent TPP exhibiting decaying excited intensity on an interval $[t_n, t_n + T_w]$. Suppose the actual decay causes the true intensity $\lambda^*(t)$ to drop strictly by a margin $\gamma > 0$ (i.e., $\lambda^*(t_n^+) - \lambda^*(t_n + T_w) = \gamma$).
For any estimated intensity $\hat{\lambda}(t)$ modeled by a given neural architecture, we define its maximum approximation error over the decaying window as:
$$
    \text{Error}(\hat{\lambda}) := \sup_{t \in [t_n, t_n+T_w]} \vert \lambda^*(t) - \hat{\lambda}(t) \vert.
$$
Let $\Lambda_{\text{MNN}}(t)$ be parameterized by any strictly convex MNN (forcing $\lambda_{\text{MNN}}'(t) \ge 0$), and let $\Lambda_{\text{MAS}}(t)$ be parameterized by MAS with a maximum knot interval $\Delta$. The approximation gap between the MNN and MAS estimators satisfies the following intrinsic lower bound:
\begin{equation}
\label{eq: MNN_MAS_gap}
    \inf_{\text{MNN}}\, \text{Error}(\lambda_{\text{MNN}}) - \text{Error}(\lambda_{\text{MAS}}) \ge \frac{\gamma}{2} - C_0 \Delta.
\end{equation}
\end{theorem}

\cref{thm: MAS vs MNN gap} reveals that, for any sufficiently fine-grained MAS configuration satisfying $\Delta < \frac{\gamma}{2 C_0}$, MAS achieves a strictly smaller approximation error than MNNs. The error gap mathematically proves that MAS breaks the inherent approximation deadlock of continuous MNNs.

\subsection{Hyperparameter Selection for MAS}

Beyond theoretical comparisons, \cref{thm: Gen Error for MAS: 2} also offers guidance on selecting hyperparameters. Given a fixed sample size $Z$, two important hyperparameters, $\Delta$ and $L$, can be chosen by minimizing the bound in \cref{thm:generalization bound}, as is presented in the following corollary. 

\begin{corollary}
\label{col: selection of L and delta}
    In \cref{thm: Gen Error for MAS: 2}, when the generalization bound is minimized, the total interpolation length $L$ and interpolation interval $\Delta$ satisfy:
    \vspace{-3pt}
    \begin{equation}
        L\ge \sqrt{\frac{\mathbb{E}[t_i] \cdot \lambda_{\text{min}}}{l_1}}, \quad \Delta\le C_{\Delta}\frac{(N_0q)^{\frac{1}{3}}\lambda_{\text{min}}^{\frac{5}{6}}(\mathbb E [t_i])^{\frac{1}{6}}}{Z^{\frac{1}{3}}l_1^{\frac{5}{6}}},
    \vspace{-3pt}
    \end{equation} 
    where $\mathbb E[t_i]$ represents the expectation of the event interval, $q$ is the number of free parameters per interpolation knot, $\lambda_{\text{min}}$ is the base intensity during $[0,T]$, $l_1$ is the Lipschitz constant of the intensity during event intervals, and $C_{\Delta}$ is a data-independent constant. Specifically, for any $\tau_1,\tau_2$ s.t. $t_{z,i}<\tau_1<\tau_2<t_{z,i+1}$, $|\lambda(\tau_1)-\lambda(\tau_2)|\le l_1|\tau_1-\tau_2|$.
\end{corollary}

According to Corollary \ref{col: selection of L and delta}, hyperparameters $L$ and $\Delta$ depend on two dataset-specific quantities: (1) the expected inter-event time $\mathbb{E}[t_i]$, and (2) the degree of temporal variability in the TPP intensity, captured by its Lipschitz constant $l_1$. As $\mathbb{E}[t_i]$ grows, a longer interpolation horizon, i.e., a larger $L$, is required to track changes in the CCIF accurately.
Moreover, as the intensity fluctuation, reflected by $l_1$, becomes sharper, MAS should provide a finer-grained approximation, which requires a smaller $\Delta$. A larger sample size $Z$ also supports a finer interpolation resolution, while a larger per-knot freedom $q$ or longer maximum sequence length $N_0$ requires a coarser interval to control complexity.
Furthermore, Corollary \ref{col: selection of L and delta} provides a simple plug-in strategy for selecting $L$ and $\Delta$:
For synthetic datasets where the functional form of the intensity is known, $l_1$, $\lambda_{\min}$, and $\mathbb{E}[t_i]$ can be estimated directly from samples, yielding an approximate lower bound for $L$ and an approximate upper bound for $\Delta$. 
For general real-world datasets, inspired by Scott’s rule \citep{hollander2013nonparametric}, we suggest assuming that the data follow a classic TPP, e.g., exponential-kernel Hawkes process, and estimating $L$ and $\Delta$ accordingly. 
Given $L$ and $\Delta$, the number of interpolation knots $p$ can also be calculated as $p=L/\Delta$.



It is worth noting that, although \cref{thm: Gen Error for MAS: 2} shows that an excessively large $L$ and an overly small $\Delta$ may lead to overfitting and thus increased generalization error, such overfitting can be effectively mitigated by standard training techniques (e.g., dropout \citep{srivastava2014dropout} or SGD-based optimization \citep{keskar2017improving}). In practice, we observe that this overfitting effect is mild. 
Therefore, we recommend using a moderately large $L$ and a relatively fine $\Delta$, consistent with the bounds provided in Corollary \ref{col: selection of L and delta}.

\section{Experiments}
\label{experiments}

\begin{figure*}[htbp]
    \captionsetup[subfloat]{font=footnotesize, labelfont=footnotesize}
    \centering
    \subfloat[Hawkes1 CIF\label{fig:hawkes1}]{
    \adjustbox{width=2.8cm,height=2.8cm}{
    \includegraphics[clip, trim=3mm 0cm 3mm 0cm, width=\linewidth]{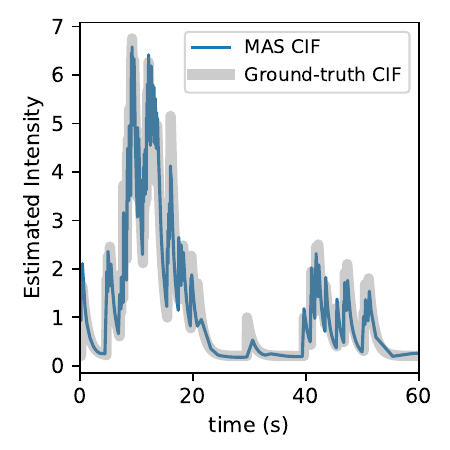}
    }
    }
    \hfill
        \subfloat[Self-correcting CIF\label{fig:sc1}]{
        \adjustbox{width=2.8cm,height=2.8cm}{
        \includegraphics[clip, trim=3mm 0cm 3mm 0cm, width=\linewidth]{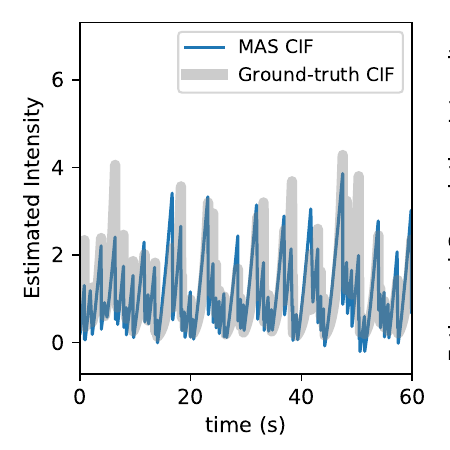}
        }
    }
    \hfill
        \subfloat[Hawkes1 CCIF\label{fig:hawkes2}]{
        \adjustbox{width=2.8cm,height=2.8cm}{
        \includegraphics[clip, trim=0mm 0cm 3mm 0cm, width=\linewidth]{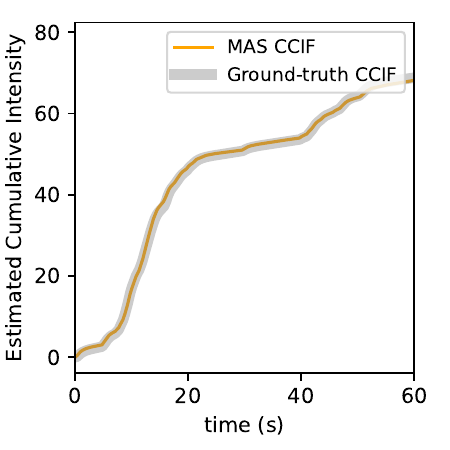}
        }
    }
    \hfill
        \subfloat[Self-correcting CCIF\label{fig:sc2}]{
        \adjustbox{width=2.8cm,height=2.8cm}{
        \includegraphics[clip, trim=0mm 1.5mm 3mm 1.5mm, width=\linewidth]{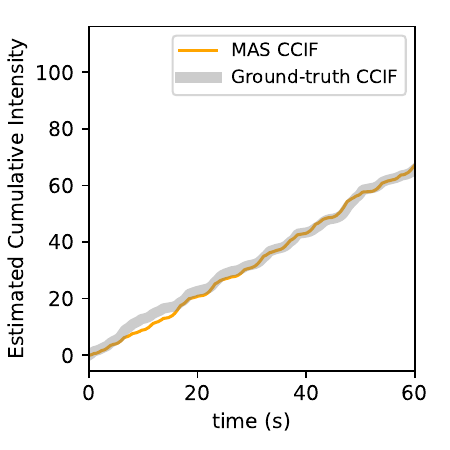}
        }
    }
    \hfill
        \subfloat[Training Time\label{fig:time}]{
        \adjustbox{width=2.8cm,height=2.8cm}{
        \includegraphics[clip, trim=2mm 0mm -2mm 0mm, width=\linewidth]{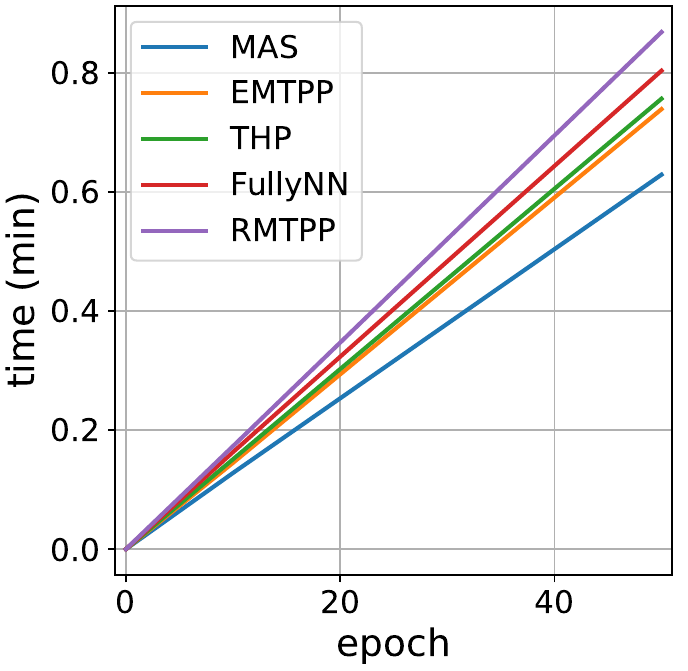}
        }
        }
\caption{Qualitative intensity estimation and efficiency of MAS. The first four panels compare the estimated and ground-truth CIF/CCIF on the Hawkes1 and Self-correcting datasets, while the last panel reports training time under matched model sizes, showing that MAS captures temporal dynamics with favorable computational efficiency.}
    \label{fig:intensity_estimation}
\end{figure*}

\begin{table*}[t]
    \centering
    \caption{Performance comparison between baselines and MAS on the univariate datasets w.r.t. NLL and RMSE. The best is highlighted in \colorbox{lightblue}{\textbf{bold}}, the second-best \colorbox{lightgreen}{\underline{underlined}}. Standard deviations are reported in parentheses.}
    \resizebox{1.0\linewidth}{!}{
    \begin{sc}
    \begin{tabular}{@{\hskip 2pt}l@{\hskip 4pt}c@{\hskip 4pt}c@{\hskip 4pt}c@{\hskip 4pt}c@{\hskip 4pt}c@{\hskip 4pt}c@{\hskip 4pt}c@{\hskip 4pt}c@{\hskip 4pt}c@{\hskip 4pt}c@{\hskip 4pt}c@{\hskip 4pt}c@{\hskip 2pt}}
        \toprule
        \multirow{2}{*}{Model} & \multicolumn{2}{c}{Hawkes1} & \multicolumn{2}{c}{Hawkes2} & \multicolumn{2}{c}{Renewal1} & \multicolumn{2}{c}{Renewal2} & \multicolumn{2}{c}{Self-correcting} \\
        \cmidrule{2-11}
        & NLL($\downarrow$) & RMSE($\downarrow$) & NLL($\downarrow$) & RMSE($\downarrow$) & NLL($\downarrow$) & RMSE($\downarrow$) & NLL($\downarrow$) & RMSE($\downarrow$) & NLL($\downarrow$) & RMSE($\downarrow$) \\
        \midrule
        IPP & -1.010\,\textsubscript{(0.001)} & 2.391\,\textsubscript{(0.019)} & -0.992\,\textsubscript{(0.006)} & 2.613\,\textsubscript{(0.005)} & -1.491\,\textsubscript{(0.002)} & 1.976\,\textsubscript{(0.025)} & -0.809\,\textsubscript{(0.001)} & {3.028}\textsubscript{(0.015)} & -1.281\,\textsubscript{(0.002)} & 1.144\,\textsubscript{(0.019)} \\

        {Hawkes} & \rval{-1.645}{0.009} & \rval{2.365}{0.002} & \rval{-1.553}{0.0092} & \rval{2.523}{0.004} & \rval{-1.336}{0.007} & \rval{2.143}{0.013} & \rval{1.487}{0.012} & \rval{3.382}{0.009} & \rval{-1.123}{0.008} & \rval{1.282}{0.011} \\

        RMTPP & -1.472\,\textsubscript{(0.023)} & 2.387\,\textsubscript{(0.006)} & -1.880\,\textsubscript{(0.021)} & 2.583\,\textsubscript{(0.004)} & -1.571\,\textsubscript{(0.039)} & 2.071\,\textsubscript{(0.012)} & -1.950\,\textsubscript{(0.030)} & 3.064\,\textsubscript{(0.007)} & -1.348\,\textsubscript{(0.021)} & 1.162\,\textsubscript{(0.001)} \\

        THP & -1.512\,\textsubscript{(0.010)} & 2.453\,\textsubscript{(0.018)} & -1.534\,\textsubscript{(0.028)} & 2.637\,\textsubscript{(0.052)} & -1.619\,\textsubscript{(0.001)} & {1.710}\textsubscript{(0.002)} & -1.508\,\textsubscript{(0.016)} & {2.980}\,\textsubscript{(0.085)} & -1.312\,\textsubscript{(0.001)} & 1.221\,\textsubscript{(0.005)} \\

        SAHP & {-1.766}\textsubscript{(0.029)} & 2.460\,\textsubscript{(0.017)} & \colorbox{lightgreen}{\underline{-2.102}}\textsubscript{(0.001)} & 2.645\,\textsubscript{(0.053)} & {-1.888}\textsubscript{(0.007)} & 1.737\,\textsubscript{(0.001)} & {-1.973}\textsubscript{(0.009)} & 2.981\,\textsubscript{(0.089)} & \colorbox{lightgreen}{\underline{-1.421}}\textsubscript{(0.001)} & 1.122\,\textsubscript{(0.019)} \\

        FullyNN & -1.471\,\textsubscript{(0.014)} & 2.463\,\textsubscript{(0.033)} & -1.488\,\textsubscript{(0.021)} & 2.651\,\textsubscript{(0.018)} & -1.648\,\textsubscript{(0.011)} & 1.929\,\textsubscript{(0.057)} & -1.453\,\textsubscript{(0.011)} & 3.064\,\textsubscript{(0.039)} & -1.374\,\textsubscript{(0.006)} & 1.220\,\textsubscript{(0.018)} \\

        TriTPP & -1.510\,\textsubscript{(0.005)} & 2.409\,\textsubscript{(0.022)} & -1.939\,\textsubscript{(0.000)} & 2.628\,\textsubscript{(0.007)} & -1.821\,\textsubscript{(0.001)} & 2.038\,\textsubscript{(0.027)} & -1.910\,\textsubscript{(0.002)} & 3.081\,\textsubscript{(0.028)} & -1.328\,\textsubscript{(0.002)} & \colorbox{lightblue}{\textbf{0.897}}\textsubscript{(0.017)} \\

        EMTPP & -1.466\,\textsubscript{(0.021)} & \colorbox{lightgreen}{\underline{2.324}}\textsubscript{(0.049)} & -1.468\,\textsubscript{(0.013)} & {2.541}\textsubscript{(0.007)} & -1.646\,\textsubscript{(0.004)} & {1.671}\textsubscript{(0.102)} & -1.453\,\textsubscript{(0.012)} & 3.038\,\textsubscript{(0.014)} & {-1.378}\textsubscript{(0.003)} & 1.201\,\textsubscript{(0.061)} \\

        IFTPP & \colorbox{lightgreen}{\underline{-1.802}}\,\textsubscript{(0.026)} & {2.429}\textsubscript{(0.032)} & -1.873\,\textsubscript{(0.007)} &{2.563}\textsubscript{(0.052)} & 
        \colorbox{lightgreen}{\underline{-1.903}}\,\textsubscript{(0.009)} & 1.789\,\textsubscript{(0.012)} &
        -1.385\,\textsubscript{(0.004)} & {2.967}\textsubscript{(0.104)} &  \colorbox{lightblue}{\textbf{-1.478}}\textsubscript{(0.002)} & {0.963}\,\textsubscript{(0.009)} \\

        WSM & -1.552\,\textsubscript{(0.0114)} & {2.438}\textsubscript{(0.015)} & -1.578\,\textsubscript{(0.016)} &{2.609}\textsubscript{(0.041)} & 
        -1.599\,\textsubscript{(0.001)} & 1.681\,\textsubscript{(0.119)} &
        -1.531\,\textsubscript{(0.016)} & 2.943\,\textsubscript{(0.116)} &  {-1.308}\textsubscript{(0.001)} & 1.219\,\textsubscript{(0.000)} \\

        {DTPP} & \rval{-1.455}{0.005} & \rval{2.951}{0.007} & \rval{-2.031}{0.012} & \rval{\colorbox{lightgreen}{\underline{2.517}}}{0.002} & \rval{-1.651}{0.007} & \rval{1.833}{0.018} & \rval{\colorbox{lightblue}{\textbf{-2.097}}}{0.011} & \rval{\colorbox{lightgreen}{\underline{2.875}}}{0.046} & \rval{-1.341}{0.003} & \rval{\colorbox{lightgreen}{\underline{0.898}}}{0.000} \\
        
       {CuFun} & \rval{-1.488}{0.006} & \rval{2.441}{0.009} & \rval{-1.543}{0.011} & \rval{2.634}{0.005} & \rval{-1.844}{0.003} & \rval{\colorbox{lightblue}{\textbf{1.441}}}{0.002} & \rval{-1.454}{0.005} & \rval{3.069}{0.023} & \rval{-1.410}{0.002} & \rval{1.220}{0.001} \\

        {HYPRO} & \rval{-1.239}{0.003} & \rval{2.361}{0.011} & \rval{-0.948}{0.005} & \rval{2.553}{0.021} & \rval{-1.524}{0.004} & \rval{1.712}{0.082} & \rval{-1.521}{0.012} & \rval{3.012}{0.012} & \rval{-1.214}{0.004} & \rval{1.192}{0.002} \\
        

        \midrule
        MAS & \colorbox{lightblue}{\textbf{-1.810}}\textsubscript{(0.024)} & \colorbox{lightblue}{\textbf{2.313}}\textsubscript{(0.025)} & \colorbox{lightblue}{\textbf{-2.190}}\textsubscript{(0.010)} & \colorbox{lightblue}{\textbf{2.488}}\textsubscript{(0.059)} & \colorbox{lightblue}{\textbf{-1.952}}\textsubscript{(0.003)} & \colorbox{lightgreen}{\underline{1.533}}\,\textsubscript{(0.178)} & \colorbox{lightgreen}{\underline{-2.062}}\textsubscript{(0.012)} & \colorbox{lightblue}{\textbf{2.823}}\textsubscript{(0.041)} & -1.380\,\textsubscript{(0.006)} & {1.215}\textsubscript{(0.054)} \\
        \bottomrule
    \end{tabular}
    \end{sc}  
    } 
    \label{tab:hawkes_results}
\end{table*}

This section presents the performance of MAS compared to other baseline models. We also evaluate the efficiency of MAS relative to other CCIF- or CIF-based methods. An ablation study is conducted at the end of the experiment. 

\subsection{Baselines, Evaluation Metrics and Experiment Setups}

\textbf{Baselines}:
For experiment baselines, we systematically categorize thirteen representative models into four paradigms:
\begin{itemize}
    \item \textit{Parametric Models:} Inhomogeneous Poisson Process (\textbf{IPP}) \citep{daley2007introduction} and the standard \textbf{Hawkes} process \citep{hawkes1971spectra}.
    \item \textit{CIF Models:} Approaches that directly parameterize the CIF, including \textbf{RMTPP} \citep{du2016recurrent}, \textbf{THP} \citep{simiao2020transformer}, and \textbf{SAHP} \citep{zhang2020SAHP}.
    \item \textit{MNN-based CCIF Models:} Approaches that leverage continuous MNNs to model the CCIF, like \textbf{FullyNN} \citep{omi2019fully}, \textbf{EMTPP} \citep{liu2024cumulative}, and \textbf{CuFun} \citep{wang2024cumulative}.
    \item \textit{Integral/Intensity-Free and other state-of-the-art Models:} Alternative flow-based or score-matching-based models, including \textbf{IFTPP} \citep{shchur2019intensity}, \textbf{TriTPP} \citep{shchur2020fast} and \textbf{WSM} \citep{cao2025scorematching}. We also compare other state-of-the-art models like \textbf{HYPRO} \citep{xue2022hypro} and \textbf{DTPP} \citep{panosdecomposable}.
\end{itemize}

Some baselines, such as IPP, {Hawkes}, FullyNN, and TriTPP, only consider the univariate case in their original implementations. Therefore, our experiments are divided into two parts. First, we compare the performance of MAS with all baselines on univariate synthetic data. Then, we evaluate MAS on multivariate real-world datasets, comparing it with multivariate baselines, including RMTPP, THP, SAHP, EMTPP, IFTPP, WSM, {CuFun, HYPRO, and DTPP}.


\textbf{Metrics}:
We use three evaluation metrics: negative log-likelihood (\textbf{NLL}), root-mean-square error (\textbf{RMSE}) of predicted timestamps, and accuracy (\textbf{ACC}) of predicted event types. 
The log-likelihood, predicted timestamps, and event types are computed based on \cref{likelihood,predicting t_n+1}. 

\textbf{Experimental Setup}:
To ensure the performance gains solely originate from our proposed spline-based CCIF formulation rather than the capacity of the history encoder, we rigorously align the encoder architecture. Specifically, MAS, FullyNN, THP, EMTPP, and WSM are all uniformly equipped with an identically configured Transformer network (1 layer, 16 attention heads, 64-dimensional model size) to extract the history embedding. 
Furthermore, since continuous MNNs (e.g., FullyNN, EMTPP) are highly sensitive to large numerical inputs and prone to encountering convergence failures during optimization, we applied a uniform timestamp scaling strategy (e.g., dividing inter-event times by 10 or 100 depending on the absolute scale of the datasets). All baselines were retrained and evaluated under the exact same scaled configurations, identical learning rate ($0.001$), fixed batch size ($64$), and identical $60\% - 20\% - 20\%$ train-validation-test splits. 
For MAS, we incorporate 10 knots for the interpolation component, using RQS as the interpolation monotone spline and a linear-plus-exponential function \(g(t)= c + bt - \exp(-at) \) as the extrapolation function. 
The interpolation support \( L \) is set to a constant. 
All models are trained with the Adam optimizer, conducted on an RTX 4090 with 24GB of memory. Details refer to Appendix C. 


\subsection{Synthetic Univariate Data}
\textbf{Datasets}:
In the synthetic univariate experiment, we consider five datasets constructed by FullyNN \citep{omi2019fully}: two Hawkes processes (1,2) with different intensities, two renewal processes 1,2 (stationary/non-stationary), and a self-correcting process. Details refer to Appendix C.

\begin{table*}[t]
    \centering
    \caption{Performance comparison between baselines and MAS on the multivariate real-world datasets w.r.t. NLL, ACC and RMSE. 
    The best is highlighted in \colorbox{lightblue}{\textbf{bold}}, the second-best \colorbox{lightgreen}{\underline{underlined}}. Standard deviations are reported in parentheses. 
    }
\resizebox{\linewidth}{!}{
     \begin{sc}
      \begin{tabular}{@{\hskip 2pt}l@{\hskip 4pt}c@{\hskip 4pt}c@{\hskip 4pt}c@{\hskip 4pt}c@{\hskip 4pt}c@{\hskip 4pt}c@{\hskip 4pt}c@{\hskip 4pt}c@{\hskip 4pt}c@{\hskip 4pt}c@{\hskip 4pt}c@{\hskip 4pt}c@{\hskip 2pt}}
        \toprule
        \multirow{2}{*}{Model} & \multicolumn{3}{c}{Retweet} & \multicolumn{3}{c}{Earthquake} & \multicolumn{3}{c}{Taxi} & \multicolumn{3}{c}{Taobao} \\
        \cmidrule{2-13}
        & NLL($\downarrow$) & ACC($\uparrow$) & RMSE($\downarrow$) & NLL($\downarrow$) & ACC($\uparrow$) & RMSE($\downarrow$) & NLL($\downarrow$) & ACC($\uparrow$) & RMSE($\downarrow$) & NLL($\downarrow$) & ACC($\uparrow$) & RMSE($\downarrow$) \\
        \midrule

        RMTPP    & {-0.468}\,\textsubscript{(0.011)} & 0.559\,\textsubscript{(0.001)} & 22.19\,\textsubscript{(0.003)} & 0.251\,\textsubscript{(0.006)} & 0.424\,\textsubscript{(0.000)} & 1.530\,\textsubscript{(0.004)} & -0.299\,\textsubscript{(0.059)} & 0.905\,\textsubscript{(0.005)} & 0.366\,\textsubscript{(0.001)} & -0.191\,\textsubscript{(0.029)} & 0.442\,\textsubscript{(0.001)} & {0.132}\,\textsubscript{(0.000)} \\

        THP      & -0.459\,\textsubscript{(0.003)} & 0.585\,\textsubscript{(0.000)} & \colorbox{lightgreen}{\underline{20.88}}\,\textsubscript{(0.004)} & 0.300\,\textsubscript{(0.010)} & 0.451\,\textsubscript{(0.000)} & \colorbox{lightgreen}{\underline{1.391}}\,\textsubscript{(0.005)} & -0.293\,\textsubscript{(0.001)} & 0.913\,\textsubscript{(0.001)} & {0.356}\,\textsubscript{(0.000)} & -0.227\,\textsubscript{(0.025)} & \colorbox{lightgreen}{\underline{0.594}}\,\textsubscript{(0.021)} & 0.132\,\textsubscript{(0.000)} \\

        SAHP     & -0.454\,\textsubscript{(0.000)} & 0.584\,\textsubscript{(0.000)} & \colorbox{lightblue}{\textbf{20.43}}\,\textsubscript{(0.019)} & 0.275\,\textsubscript{(0.012)} & 0.417\,\textsubscript{(0.003)} & {1.448}\,\textsubscript{(0.001)} & {-0.442}\,\textsubscript{(0.077)} & 0.902\,\textsubscript{(0.005)} & 0.358\,\textsubscript{(0.002)} & {-0.586}\,\textsubscript{(0.099)} & 0.572\,\textsubscript{(0.011)} & 0.141\,\textsubscript{(0.003)} \\

        EMTPP & -0.446\,\textsubscript{(0.013)} & {0.595}\,\textsubscript{(0.001)} & 22.10\,\textsubscript{(0.075)} & 0.388\,\textsubscript{(0.009)} & {0.458}\,\textsubscript{(0.002)} & 1.517\,\textsubscript{(0.006)} & -0.312\,\textsubscript{(0.025)} & {0.915}\,\textsubscript{(0.003)} & 0.362\,\textsubscript{(0.001)} & -0.039\,\textsubscript{(0.005)} & 0.591\,\textsubscript{(0.008)} & {0.131}\,\textsubscript{(0.001)} \\

        IFTPP & \colorbox{lightgreen}{\underline{-0.477}}\,\textsubscript{(0.015)} & \colorbox{lightblue}{\textbf{0.603}}\,\textsubscript{(0.003)} & 22.18\,\textsubscript{(0.204)} & \colorbox{lightgreen}{\underline{0.191}}\,\textsubscript{(0.034)} & {0.434}\,\textsubscript{(0.013)} & 1.488\,\textsubscript{(0.007)} & -0.206\,\textsubscript{(0.019)} & {0.914}\,\textsubscript{(0.006)} & 0.377\,\textsubscript{(0.003)} & {-0.594}\,\textsubscript{(0.013)} & 0.446\,\textsubscript{(0.005)} & {0.134}\,\textsubscript{(0.003)} \\

        WSM & -0.449\,\textsubscript{(0.002)} & {0.599}\,\textsubscript{(0.000)} & 22.12\,\textsubscript{(0.017)} & {0.235}\,\textsubscript{(0.002)} & \colorbox{lightblue}{\textbf{0.470}}\,\textsubscript{(0.000)} & 1.510\,\textsubscript{(0.003)} & -0.281\,\textsubscript{(0.001)} & \colorbox{lightgreen}{\underline{0.916}}\,\textsubscript{(0.001)} & 0.361\,\textsubscript{(0.001)} & \colorbox{lightblue}{\textbf{-1.115}}\,\textsubscript{(0.053)} & 0.581\,\textsubscript{(0.005)} & {0.132}\,\textsubscript{(0.000)} \\

        {DTPP} & \rval{-0.459}{0.002} & \colorbox{lightgreen}{\underline{0.601}}\scriptsize{(0.001)} & \rval{22.38}{0.012} & \rval{0.256}{0.015} & \rval{0.462}{0.002} & \rval{1.405}{0.041} & \rval{\colorbox{lightgreen}{\underline{-0.446}}}{0.012} & \rval{\colorbox{lightblue}{\textbf{0.923}}}{0.001} & \rval{0.361}{0.007} & \rval{-0.984}{0.003} & \rval{0.593}{0.005} & \rval{\colorbox{lightgreen}{\underline{0.128}}}{0.000} \\

        {CuFun} & \rval{-0.462}{0.001} & \rval{0.600}{0.002} & \rval{20.96}{0.008} & \rval{0.230}{0.012} & \rval{0.463}{0.003} & \rval{1.430}{0.014} & \rval{-0.273}{0.008} & \rval{0.914}{0.002} & \colorbox{lightgreen}{\underline{0.345}}\scriptsize{(0.004)} & \rval{-0.878}{0.012} & \rval{0.585}{0.005} & \rval{0.133}{0.003} \\

        {HYPRO} & \rval{-0.411}{0.023} & \rval{0.599}{0.002} & \rval{21.24}{0.003} & \rval{0.743}{0.032} & \rval{0.456}{0.001} & \rval{1.452}{0.012} & \rval{-0.428}{0.011} & \rval{{0.914}}{0.005} & \rval{0.363}{0.003} & \rval{-0.947}{0.032} & \rval{0.583}{0.004} & \rval{0.130}{0.002} \\

        \midrule
        MAS      & \colorbox{lightblue}{\textbf{-0.489}}\,\textsubscript{(0.044)} & {0.598}\,\textsubscript{(0.005)} & {22.15}\,\textsubscript{(0.673)} & \colorbox{lightblue}{\textbf{0.186}}\,\textsubscript{(0.007)} & \colorbox{lightblue}{\textbf{0.470}}\,\textsubscript{(0.004)} & \colorbox{lightblue}{\textbf{1.217}}\,\textsubscript{(0.013)} & \colorbox{lightblue}{\textbf{-0.479}}\,\textsubscript{(0.003)} & {0.905}\,\textsubscript{(0.002)} & \colorbox{lightblue}{\textbf{0.331}}\,\textsubscript{(0.001)} & \colorbox{lightgreen}{\underline{-1.005}}\,\textsubscript{(0.055)} & \colorbox{lightblue}{\textbf{0.596}}\,\textsubscript{(0.021)} & \colorbox{lightblue}{\textbf{0.126}}\,\textsubscript{(0.003)} \\
        \bottomrule
    \end{tabular}
    \end{sc}
     }
     
    \label{tab:performance_comparison}
\end{table*}

\textbf{Results}:
\cref{tab:hawkes_results} presents the performance comparison between all baselines and MAS. 
MAS achieved the best performance across most metrics on the majority of datasets. This validates our theory that, by introducing piecewise monotone spline as the interpolation component, MAS is capable of fitting various ground-truth CCIFs derived from different classical TPP. 
In contrast, other traditional models (e.g., IPP) or deep Hawkes processes (e.g., THP) become less effective in characterizing other TPP. 
Meanwhile, MAS also outperforms models that parametrize CCIF using MNNs in terms of fitting ability. \cref{fig:intensity_estimation} shows MAS’s predicted CIF and CCIF for two datasets, implying MAS successfully captures the intensity patterns of different TPP. 
A detailed comparison between MAS and FullyNN is provided in Appendix C, further validating the superiority of MAS over classic CCIF-based methods.

\subsection{Real Multivariate Data}
\textbf{Datasets}:
We consider four real multivariate datasets: \textbf{Taxi} \cite{TAXI}, \textbf{Taobao} \cite{xue2023easytpp}, \textbf{Retweet} \cite{RETWEET} and \textbf{Earthquake} \cite{xue2023easytpp}. Detailed information is in Appendix C.

\textbf{Results}:
\cref{tab:performance_comparison} presents the performance comparison between all baselines and MAS. 
On real-world datasets, MAS once again achieved the best results across multiple metrics. This further demonstrates the advantages of MAS in terms of flexibility and generalization, indicating that in multivariate scenarios, MAS successfully captures the interactions among multiple variables.
The success of MAS can be attributed to two factors: first, the inclusion of more interpolation knots along the timeline, which makes MAS more flexible than baseline models; second, the introduction of the extrapolation component, which ensures that MAS correctly simulates the CCIF while maintaining strong generalization ability. Therefore, our model achieves greater flexibility compared to baseline models without compromising generalization.

\subsection{Efficiency Test}
We further compare the computational efficiency of MAS against THP, RMTPP, and two CCIF-based models, EMTPP and FullyNN. To ensure a fair comparison, we roughly align the parameter counts of all models on the Hawkes1 dataset and record the actual training time over 50 epochs under identical settings. As shown in \cref{fig:time}, MAS runs significantly faster than all four baselines, demonstrating that it is highly competitive in computational efficiency alongside its modeling accuracy. The details are provided in Appendix C.
For EMTPP and FullyNN, automatic differentiation invariably increases the per-step computational overhead and exacerbates backpropagation costs over long sequences and high-dimensional parameter spaces. 
In contrast, the explicitly decoupled piecewise interpolation and extrapolation components of MAS enable direct intensity evaluations within each time segment, eliminating the need for gradient flows or numerical approximations. Ultimately, MAS delivers significantly lower training times and superior computational efficacy.

\subsection{Long Horizon Prediction}
We further evaluate the long-horizon prediction capabilities of MAS on the univariate Hawkes1 and multivariate TAXI datasets, comparing its performance against THP, EMTPP, and CuFun in forecasting the next 3, 5, and 10 events. As shown in \cref{tab:long_horizon}, MAS not only maintains high accuracy in short-term predictions but also consistently outperforms the baselines as the forecasting horizon extends. 
The sustained superiority of MAS in multi-step scenarios indicates that it does not merely overfit localized intensity fluctuations immediately following an event; rather, it effectively leverages the extrapolation component to secure stable long-term trends, thereby significantly mitigating cumulative error propagation. For sequences exhibiting strong self-exciting dynamics, such as Hawkes1, MAS precisely captures both the rapid intensity decay and subsequent re-excitations. Meanwhile, on multivariate datasets like TAXI, MAS seamlessly integrates shared history representations with highly expressive piecewise monotone splines, effectively modeling complex mutual interactions across different event types and consequently enhancing the joint forecasting accuracy for multiple future events.

\begin{table}[t]
    \centering
    \caption{Long-horizon results. The best is in \colorbox{lightblue}{\textbf{bold}}.}
    \begingroup
    \resizebox{0.9\linewidth}{!}{
    \begin{sc}
    \begin{tabular}{c c c c c}
        \toprule
        & & \multicolumn{2}{c}{TAXI} & HAWKES1 \\
        \cmidrule(lr){3-4} \cmidrule(lr){5-5}
        \# Events & Model & ACC($\uparrow$) & RMSE($\downarrow$) & RMSE($\downarrow$) \\
        \midrule
        \multirow{4}{*}{\textbf{Next 3}} 
        & THP   & 0.818 & 0.585 & 3.597 \\
        & EMTPP & 0.818 & 0.588 & 3.921 \\
        & CuFun & 0.802 & \colorbox{lightblue}{\textbf{0.523}} & 3.833 \\
        & MAS   & \colorbox{lightblue}{\textbf{0.819}} & 0.525 & \colorbox{lightblue}{\textbf{3.312}} \\

        \midrule
        \multirow{4}{*}{\textbf{Next 5}} 
        & THP   & \colorbox{lightblue}{\textbf{0.796}} & 0.832 & 5.445 \\
        & EMTPP & 0.794 & 0.740 & 5.129 \\
        & CuFun & 0.790 & 0.758 & 5.146 \\
        & MAS   & 0.791 & \colorbox{lightblue}{\textbf{0.704}} & \colorbox{lightblue}{\textbf{4.281}} \\

        \midrule
        \multirow{4}{*}{\textbf{Next 10}} 
        & THP   & 0.780 & 1.338 & 8.356 \\
        & EMTPP & 0.750 & 1.173 & 7.995 \\
        & CuFun & 0.772 & 1.341 & 7.797 \\
        & MAS   & \colorbox{lightblue}{\textbf{0.785}} & \colorbox{lightblue}{\textbf{1.005}} & \colorbox{lightblue}{\textbf{6.235}} \\
        \bottomrule
    \end{tabular}
    \end{sc}
    }
    \endgroup
    \label{tab:long_horizon}
\end{table}

\subsection{Ablation Study}
\textbf{Parameter Sensitivity}: We conduct ablation studies to verify the contribution of each component in MAS. We consider the following factors: \textbf{(1)} the type of interpolation monotone spline (RQS/RLS), \textbf{(2)} the type of extrapolation function (linear/linear-plus-exponential), \textbf{(3)} the interpolation support length ($L=4/6/10$), \textbf{(4)} the number of interpolation knots ($p=5/10/50/100$). 
The experiments are performed on the Taxi and Earthquake datasets, using NLL and ACC as metrics. We set RQS interpolation, linear-plus-exponential extrapolation, \( L = 6 \), and \( p = 10 \) as the default choices. 

\begin{table}[t]
    \centering
    \caption{Ablation study results. The best is in \colorbox{lightblue}{\textbf{bold}}.}
    \resizebox{0.9\linewidth}{!}{
    \begin{sc}
    \begin{tabular}{c c c c c c}
        \toprule
        & & \multicolumn{2}{c}{TAXI} & \multicolumn{2}{c}{EARTHQUAKE} \\
        \cmidrule(lr){3-4} \cmidrule(lr){5-6}
        Component & Setting & NLL($\downarrow$) & ACC($\uparrow$) & NLL($\downarrow$) & ACC($\uparrow$) \\
        \midrule
        
        \multirow{4}{*}{\shortstack{\# Knots \\ ($p$)}} 
        & p=5   & -0.460   & 0.899 & 0.185  & 0.467  \\
        & p=10  & \colorbox{lightblue}{\textbf{-0.479}}  & \colorbox{lightblue}{\textbf{0.905}} & 0.183  & \colorbox{lightblue}{\textbf{0.470}} \\
        & p=50  & -0.467 & 0.901 & \colorbox{lightblue}{\textbf{0.180}}   & 0.464 \\
        & p=100 & -0.447 & 0.899 & 0.209  & 0.441 \\
        
        \midrule
        \multirow{3}{*}{\shortstack{Length \\ ($L$)}} 
        & L=4  & {-0.427} & 0.897 & 0.190   & 0.457  \\
        & L=6  & \colorbox{lightblue}{\textbf{-0.479}} & \colorbox{lightblue}{\textbf{0.905}} & 0.186 & \colorbox{lightblue}{\textbf{0.470}}  \\
        & L=10  & -0.424 & 0.904 & \colorbox{lightblue}{\textbf{0.172}} & {0.466} \\
        
        \midrule
        \multirow{2}{*}{Extrap.} 
        & exp    & \colorbox{lightblue}{\textbf{-0.479}} & \colorbox{lightblue}{\textbf{0.905}} & 0.186 & \colorbox{lightblue}{\textbf{0.470}} \\
        & linear & -0.460  & 0.900  & \colorbox{lightblue}{\textbf{0.180}}  & 0.465  \\
        
        \midrule
        \multirow{2}{*}{Interp.} 
        & RQS & \colorbox{lightblue}{\textbf{-0.479}} & 0.905 & \colorbox{lightblue}{\textbf{0.186}} & \colorbox{lightblue}{\textbf{0.470}}  \\
        & RLS & -0.470  & \colorbox{lightblue}{\textbf{0.910}}  & {0.186} & {0.466}\\
        \bottomrule
    \end{tabular}
    \end{sc}
    }
    \label{tab:ablation_study}
\end{table}

As shown in \cref{tab:ablation_study}, as the number of interpolation knots \( p \) increases and the interpolation support length \( L \) extends, the model's performance on both datasets first improves and then declines. This validates our theoretical analysis: the increase in interpolation support length and the number of interpolation knots leads to a rise in the complexity error term in \cref{thm: Gen Error for MAS: 2}. This overfitting, however, is not severe. 
Furthermore, MAS is insensitive to different interpolation and extrapolation functions. We speculate that under the default settings (L=6, p=10), MAS already fits the CCIFs at sufficiently fine granularity, making different functions yield negligible differences.

\section{Conclusion}
\label{conclusion}
We propose a novel approach, MAS, for modeling the CCIF of TPPs. Compared to MNNs, MAS offers greater generality, flexibility, and efficiency, and enjoys strong theoretical guarantees for both fitting and generalization.
We prove that the interpolation component of MAS enhances fitting performance, while the extrapolation component of MAS enhances generalization performance. 
Extensive experiments demonstrate that MAS outperforms other CCIF- and CIF-based modeling methods on most univariate and multivariate benchmarks across multiple evaluation metrics.

\bibliographystyle{IEEEtran}
\bibliography{11_references}


\onecolumn
{
\appendices
\section{Notations}
\label{app: notation}
All notations used in the paper are listed in the \cref{tab: notation}. 

\begin{table*}[ht]
\caption{List of notations.}
\begin{tabularx}{\linewidth}{lX|lX}
\toprule
Symbol & Definition & Symbol & Definition \\
\midrule
$a_m,b_m,\cdots,e_m$ & parameters of $f_m(t)$ & $C_1,\cdots,C_4,B$ & generic constants \\
$\delta_m, y_m, w_m$ & derivative, cumulative value, and knot location at the $m$-th interpolation point & $\Delta$ & maximum interpolation interval length \\
$f_m(t)$ & $m$-th monotone spline interpolation segment & $f_{m,\theta}^{*(n)}$ & $m$-th MAS interpolation after $t_n$, controlled by $\theta$ \\
$g(t)$ & an extrapolation function & $g_{\theta}^{*(n)}$ & MAS extrapolation after $t_n$, controlled by $\theta$ \\
$\boldsymbol{h}_n$ & history embedding & $k_n$ & $n$-th event type \\
$L$ & total interpolation support length & $\lambda^*_k$ & ground-truth conditional intensity function \\
$\lambda^{*(n)}$ & ground-truth conditional intensity after event $n$ & $\Lambda^*_k$ & conditional cumulative intensity for the $k$-th variable \\
$\Lambda^{*(n)}$ & ground-truth conditional cumulative intensity after $t_n$ & $\Lambda_{\theta}^{*}$ & estimated cumulative intensity, controlled by $\theta$ \\
$\mathcal{L}$ & expected negative log-likelihood & $\widehat{\mathcal{L}}$ & empirical negative log-likelihood \\
${\mathcal{L}}^*$ & optimal negative log-likelihood induced by the ground-truth CCIF & $M$ & number of interpolation intervals or knots \\
$N$ & the number of timestamps in a sequence & $N_0$ & maximum sequence length \\
$p(\cdot)$ & probability measure & $S$ & a TPP sequence \\
$T$ & Time window & $t_n$ & $n$-th timestamp \\
$\theta_1$ & parameters in the encoder & $\theta_2$ & parameters in the MLP \\
$\widehat{t}_{n+1}$ & predicted event timestamp & $\widehat{k}_{n+1}$ & predicted event type \\
$\mathcal{H}_{t^-}$ & past history & $Z$ & number of sequences \\
\bottomrule
\end{tabularx}
\label{tab: notation}
\end{table*}

\section{Monotone Splines}
\label{app: monotone splines}

This appendix section introduces some common monotone splines and discusses their parameterizations.

\subsection{Rational Quadratic Splines}
\label{appen: RQS}

A piece of RQS on the interval $[w_{m-1},w_m]$ can be written as:
\begin{equation*}
f_{m}(t)= y_{m-1} + \frac{\left(y_{m} - y_{m-1}\right) \left(s_m \tau^2 + \delta_{m-1} \tau (1 - \tau)\right)}{s_{m} + \left(\delta_{m} + \delta_{m-1} - 2s_{m}\right) \tau (1 - \tau)}, 
\label{RQS_another form}
\end{equation*}
where $s_m=\frac{y_{m}-y_{m-1}}{w_m-w_{m-1}}$ and $\tau=\frac{t-w_{m-1}}{w_m-w_{m-1}}\in[0,1]$. When $t=w_{m-1}$, i.e., $\tau=0$, $f_{m}(t)=y_{m-1}$, $f^{\prime}_{m}(t)=\delta_{m-1}$. When $t=w_{m}$, i.e., $\tau=1$, $f_{m}(t)=y_{m}$, $f^{\prime}_{m}(t)=\delta_{m}$. This implies that once $w_{m-1}, w_{m}, \delta_{m-1},\delta_{m},y_{m-1},y_{m}$ are specified, the RQS piece is uniquely determined.

The derivative with respect to $\tau$ can be written as:
\begin{equation*}
    \begin{aligned}
    f_m'(\tau)
&= (y_m - y_{m-1})
\frac{
  \bigl[2s_m\tau + \delta_{m-1}(1-2\tau)\bigr]
  \;
}{
  J_m
}\\
&-(y_m - y_{m-1})\frac{\;
  \bigl[s_m\tau^2 + \delta_{m-1}\,\tau(1-\tau)\bigr]
  V_m(1-2\tau)}{J_m^2},
\end{aligned}
\end{equation*}
where $V_m=\delta_m+\delta_{m-1}-2s_m$, $J_m=\bigl[s_m + V_m\,\tau(1-\tau)\bigr]$. The time derivative is obtained by multiplying the expression above by $\frac{d\tau}{dt}=\frac{1}{w_m-w_{m-1}}$.

\subsection{Rational Cubic Splines}

A piece of RCS on the interval $[w_m,w_{m+1}]$ can be written as:
\begin{equation}
\begin{aligned}
f_m(t)  &= 
\frac{(1 - \tau)^3 v_m y_m + \tau(1 - \tau)^2 \left[ (2u_m v_m + v_m) y_m + v_m h_m d_m \right] 
}{(1 - \tau)^2 v_m + 2 u_m v_m \tau(1 - \tau) + \tau^2 u_m}\notag\\
&+\frac{\tau^2(1 - \tau)\left[ (2u_m v_m + u_m) y_{m+1} - u_m h_m \delta_{m+1} \right] 
+ \tau^3 u_m y_{m+1}}{(1 - \tau)^2 v_m + 2 u_m v_m \tau(1 - \tau) + \tau^2 u_m}.
\end{aligned}
\end{equation}
where $\tau=\frac{t-w_m}{w_{m+1}-w_m}\in[0,1]$ and $h_m=w_{m+1}-w_m$,
$f_{m}(w_m)=y_m$, 
$f_{m}(w_{m+1})=y_{m+1}$, 
$f'_{m}(w_m)=\delta_m$, 
$f'_{m}(w_{m+1})=\delta_{m+1}$. $u_m$ and $v_m$ are two shape parameters, $u_m,v_m>0$.

\subsection{Monotone Cubic Splines}

On $[w_m, w_{m+1}]$, let $h_m = w_{m+1}-w_m,\ \tau = \frac{ t- w_m}{h_m}$, then a piece of monotone cubic spline can be expressed as:
\begin{align*}
    f_m(t)
&= y_m\,h_{00}(\tau)
+ h_m\,\delta_m\,h_{10}(\tau)
\\
&+ y_{m+1}\,h_{01}(\tau)
+ h_m\,\delta_{m+1}\,h_{11}(\tau),
\end{align*}
where
\[
\begin{aligned}
h_{00}(\tau) = 2\tau^3 - 3\tau^2 + 1,\ &
h_{10}(\tau) = \tau^3 - 2\tau^2 + \tau,\\
h_{01}(\tau) = -2\tau^3 + 3\tau^2,\ &
h_{11}(\tau) = \tau^3 - \tau^2.
\end{aligned}
\]

Again, 
$f_{m}(w_m)=y_m$, 
$f_{m}(w_{m+1})=y_{m+1}$, 
$f'_{m}(w_m)=\delta_m$, 
$f'_{m}(w_{m+1})=\delta_{m+1}$.

\subsection{Rational Linear Splines}

On the interval $[w_{m-1},w_{m}]$, a piece of RLS can be written as:

\begin{equation}\label{eq:2.1}
f(t) \;=\; a_m \;+\; \frac{c_m\,(t - w_m)}{1 + d_m\,(t - w_m)},
\quad t \in [w_{m-1},\,w_m].
\end{equation}

$f(t)$ is monotone increasing as long as $c_m>0$ and the denominator remains positive on the interval.

\section{Experiments}
\label{app: experiment}
\subsection{Datasets}
We consider the following five univariate datasets:

\textbf{(1) Hawkes Process (1, 2)}: We consider two exponential-decay Hawkes processes with CIF: $\lambda^*(t)=0.2+\sum_{t_n<t}\sum_{i=1}^M\alpha_i\exp(-\beta_i(t-t_n))$. We set $M=1,\alpha_1=0.8,\beta_1=1$ for the first dataset and $M=2,\alpha_1=0.4,\beta_1=\beta_2=1,\alpha_2=0.4$ and $\beta_2=20$ for the second dataset. 

\textbf{(2) Self-correcting Process}: The CIF is given as $\lambda^*(t)=\exp(t-\sum_{t_n<t}1)$. 

\textbf{(3) Renewal Process (Stationary / Non-stationary)}: For the stationary case, timestamp intervals are independent, log-normal distributed with a mean $1.0$ and standard deviation $0.5$. For the non-stationary case, the trend function is set to $r(t)=0.99\sin(2\pi t/20000)+1$. 

For multivariate datasets, their information is listed as follows:

\textbf{(1) Taxi} \cite{TAXI}: This dataset tracks the time-stamped taxi pick-up and drop-off events across the New York City. The event types are categorized into $K=10$ types and there are $2000$ sequences with an average sequence length of $40$. 
    
\textbf{(2) Taobao} \cite{xue2023easytpp}: This dataset contains time-stamped user click behaviors on Taobao shopping pages. The event types are categorized into $K=20$ types and there are $4800$ sequences with an average sequence length of $150$. 
    
\textbf{(3) Retweet} \cite{RETWEET}: This dataset contains time-stamped user retweet event sequences. The events are categorized into $K = 3$ types. There are $9000$ sequences with an average sequence length of $90$. 
    
\textbf{(4) Earthquake} \cite{xue2023easytpp}: This dataset contains time-stamped earthquake event sequences. The events are categorized into $K = 7$ types. There are $3000$ sequences with an average sequence length of $20$.

\subsection{Experimental Setup}
Because MNN is sensitive to the scale of its outputs, we applied a timestamp transformation to part of our data. Following the approach of \citep{cao2025scorematching}, we divided every timestamp in the retweet dataset by 100 and divided every timestamp in the earthquake dataset by 5. In addition, we divided the timestamps of each of the five synthetic datasets by 10. This ensures that all models converge within 100 epochs.

All models were trained on a single NVIDIA RTX$\cdot$4090$\cdot$D (24 GB). For fairness, MAS, FullyNN, EMTPP, THP, and WSM all encode historical information using Transformers of the same size. The Transformer hyperparameters are as follows: number of attention heads = 16, number of layers = 1, model dimension = 64, inner feed‑forward dimension = 8, key dimension = 16, value dimension = 16, and dropout rate = 0.1. For MAS, when using a Rational Quadratic Spline (RQS) as the interpolation function, we set the lower bound for both the interpolation interval width and length to 0.01, and the minimum derivative at interpolation points to 0.01.

We fixed the batch size at 64 and the learning rate at 0.001 for all models. Each model was trained for 100 epochs on the train set and converged within 100 epochs. We then evaluated their performance on the test set.

When comparing the accuracy of MAS with that of other models, we observed that—even for the same model on the same dataset—the reported performance can vary dramatically across different studies. 
For common baseline models such as SAHP, THP, and RMTPP, we referenced their best-reported accuracy from the literature (for example, \citep{2024_ICML_NJDTPP}, \citep{xue2023easytpp}, and \citep{cao2025scorematching}) and our experiments, and compared those to our MAS. 
For instance, on the Taobao dataset, THP achieved an accuracy of 0.467 according to \citep{2024_ICML_NJDTPP}, 0.531 according to \citep{xue2023easytpp}, and 0.594 in \citep{cao2025scorematching}. 
all below MAS’s performance of 0.600. 

For datasets like Taobao, where the raw time intervals are very large, we first scaled them down by a factor of 1/100 to ensure rapid convergence of models such as EMTPP. However, this scaling can cause the model’s NLL to change substantially, making it difficult to compare against past work (since they may have used different scaling factors). Therefore, we retrained all baseline models on the same scaled dataset—keeping the data, training and testing strategy, and number of epochs identical—and then reported their NLL on the same test set.

\subsection{Supplementary Results}
\label{subsec:cross val}
\subsubsection{Cross Validation}

\begin{table}[t]

    \centering
    \caption{Cross-validation results.}
    \label{tab: CrossVal}
    
    \begingroup
    
    \begin{minipage}[b]{0.45\textwidth}
        \centering
        {
            \begin{tabular}{lcc}
                \toprule
                \multicolumn{3}{c}{\textbf{Softplus Activation (Size 800)}} \\
                Configuration & NLL & RMSE \\
                \midrule
                \multicolumn{3}{l}{\textit{FullyNN (positive MLP layers)}} \\
                1 layer MLP size 16 & -0.774 & 2.347 \\
                1 layer MLP size 32 & -0.782 & 2.378 \\
                2 layer MLP size 16 & -1.452 & 2.462 \\
                2 layer MLP size 32 & -1.399 & 2.463 \\
                3 layer MLP size 16 & -1.371 & 2.461 \\
                3 layer MLP size 32 & -1.282 & 2.459 \\
                \midrule
                MAS & -1.807 & 2.302 \\
                \bottomrule
            \end{tabular}
        }
    \end{minipage}%
    \hspace{0.4cm}
    \begin{minipage}[b]{0.45\textwidth}
        \centering
        {
            \begin{tabular}{lcc}
                \toprule
                \multicolumn{3}{c}{\textbf{Sigmoid Activation (Size 800)}} \\
                Configuration & NLL & RMSE \\
                \midrule
                \multicolumn{3}{l}{\textit{FullyNN (positive MLP layers)}} \\
                1 layer MLP size 16 & -0.774 & 2.347 \\
                1 layer MLP size 32 & -0.782 & 2.378 \\
                2 layer MLP size 16 & -1.437 & 2.453 \\
                2 layer MLP size 32 & -1.449 & 2.449 \\
                3 layer MLP size 16 & -1.443 & 2.469 \\
                3 layer MLP size 32 & -1.418 & 2.458 \\
                \midrule
                MAS & -1.807 & 2.302 \\
                \bottomrule
            \end{tabular}
        }
    \end{minipage}

    \vspace{10pt}
    
    \begin{minipage}[b]{0.45\textwidth}
        \centering
        {
            \begin{tabular}{lcc}
                \toprule
                \multicolumn{3}{c}{\textbf{Softplus Activation (Size 400)}} \\
                Configuration & NLL & RMSE \\
                \midrule
                \multicolumn{3}{l}{\textit{FullyNN (positive MLP layers)}} \\
                1 layer MLP size 16 & -0.778 & 2.317 \\
                1 layer MLP size 32 & -0.779 & 2.318 \\
                2 layer MLP size 16 & -1.431 & 2.464 \\
                2 layer MLP size 32 & -1.389 & 2.446 \\
                3 layer MLP size 16 & -1.327 & 2.458 \\
                3 layer MLP size 32 & -1.200 & 2.475 \\
                \midrule
                MAS & -1.779 & 2.297 \\
                \bottomrule
            \end{tabular}
        }
    \end{minipage}%
    \hspace{0.4cm}
    \begin{minipage}[b]{0.45\textwidth}
        \centering
        {
            \begin{tabular}{lcc}
                \toprule
                \multicolumn{3}{c}{\textbf{Sigmoid Activation (Size 400)}} \\
                Configuration & NLL & RMSE \\
                \midrule
                \multicolumn{3}{l}{\textit{FullyNN (positive MLP layers)}} \\
                1 layer MLP size 16 & -0.778 & 2.317 \\
                1 layer MLP size 32 & -0.779 & 2.318 \\
                2 layer MLP size 16 & -1.347 & 2.448 \\
                2 layer MLP size 32 & -1.404 & 2.490 \\
                3 layer MLP size 16 & -1.311 & 2.497 \\
                3 layer MLP size 32 & -1.325 & 2.425 \\
                \midrule
                MAS & -1.779 & 2.297 \\
                \bottomrule
            \end{tabular}
        }
    \end{minipage}
    
    \endgroup
\end{table}

To rigorously demonstrate that MAS indeed outperforms other baselines, we conduct a cross-validation study comparing MAS with a classical CCIF-based method, FullyNN. We evaluate FullyNN under different network depths, widths, activation functions, and training data sizes to ensure that MAS’s superiority is not merely due to accidental hyperparameter choices.

Specifically, we fix the encoder architecture of MAS and FullyNN to be identical (i.e., both models receive the same input features). Following the data generation procedure of the hawkes1 dataset, we independently sample 400 and 800 sequences within the time range $[0,100]$ for training. During training, we vary FullyNN’s number of fully connected layers from 1 to 3, its hidden widths between 16 and 32, and its activation functions between softplus and sigmoid (representing different levels of smoothness). FullyNN is trained for 100 epochs until convergence. We then compare its NLL and RMSE with those of MAS, as shown in \cref{tab: CrossVal}.

The results show that when the network depth exceeds two layers, the performance of FullyNN no longer improves significantly with additional depth. Moreover, regardless of whether softplus or sigmoid is used, MAS consistently outperforms FullyNN across all configurations. This cross-validation experiment further confirms the advantage of MAS over classical CCIF-based modeling methods.
\color{black}

\subsubsection{Training Time}
\label{subsec:training time}
We align the parameter sizes of five models—THP, RMTPP, MAS, EMTPP, and FullyNN—and train them for 100 epochs on the Hawkes1 dataset. Experiments are rerun three times under the same conditions. We then measure the mean and the standard deviation of their running time (in minutes), and the results are summarized in \cref{running_time}. A visualized version of this comparison is shown in \cref{fig:time}. 
\color{black}
\begin{table}[t]
    \centering
    \caption{Comparison of running time.} 
    \label{running_time}
    \begingroup
    
    \begin{tabular}{lccccc}
        \toprule
                  & MAS   & EMTPP & THP     & FullyNN & RMTPP   \\
        \midrule
        10 epochs & 0.128 \,\textsubscript{(0.001)} & 0.145 \,\textsubscript{(0.000)} & 0.151 \,\textsubscript{(0.001)}   & 0.163 \,\textsubscript{(0.001)}   & 0.173 \,\textsubscript{(0.001)}   \\
        50 epochs & 0.629 \,\textsubscript{(0.001)} & 0.739 \,\textsubscript{(0.001)} & 0.756 \,\textsubscript{(0.001)} & 0.803 \,\textsubscript{(0.001)} & 0.868 \,\textsubscript{(0.002)} \\
        \bottomrule
    \end{tabular}
    
    \endgroup
\end{table}

\section{Proof}

The following mild assumptions are required to prove the theorems in the paper. Compared with the original version, we only make explicit the regularity conditions that were implicitly used in the derivations.

\begin{assumption}[$C^1$-Lipschitz Continuity]
    The ground-truth CCIF and MAS interpolation possess Lipschitz continuous derivatives with respect to time $t$ and parameter $\theta$ on any fixed closed interpolation interval between timestamps. The Lipschitz constants are $l_1$ and $l_2$, respectively.
    \label{assum: C-1 for der}
\end{assumption}
\begin{assumption}[Positive Derivative]
    The ground-truth CCIF and MAS interpolation possess uniformly positive derivatives on the observation window: there exists $v_1>0$ such that ${\Lambda}^{*\prime}_\theta(t),{\Lambda}^{*\prime}(t)\ge v_1$.
    \label{assum: positive for der}
\end{assumption}
\begin{assumption}[Bounded Loss and Finite Tail Moment]
    The parameter space $\Theta$ is compact. The per-event negative log-likelihood gap between the ground-truth model and MAS is bounded by $B_0$, and the inter-event time has a finite first moment $\mathbb{E}[t_i]<\infty$.
    \label{assum: bounded loss}
\end{assumption}

\subsection{Proof of Theorem~\ref{MNN Monotonicity}}
\begin{proof}
By the chain rule, the derivative of the MNN is
\[
    \frac{\partial \Lambda^*(t)}{\partial t}
    =
    \boldsymbol{W}^{(L)}
    \left(
    \prod_{l=L-1}^{1}
    \mathrm{diag}\!\left[\sigma'(\boldsymbol{z}^{(l)})\right]
    \boldsymbol{W}^{(l)}
    \right).
\]
If all weights are non-negative and $\sigma'(x)\ge 0$, every factor in the product has non-negative entries. Hence $\lambda^*(t)=\partial\Lambda^*(t)/\partial t\ge0$.
\end{proof}

\subsection{Proof of Theorem~\ref{thm: convexity trap}}
\begin{proof}
We prove the result by induction over layers, together with the non-decreasing activation condition in \cref{MNN Monotonicity}. For the first layer, each hidden unit has the form $\sigma(wt+b)$ with $w\ge0$. Since $\sigma$ is convex and non-decreasing, $\sigma(wt+b)$ is convex in $t$. Suppose the units in layer $l-1$ are convex functions of $t$. A non-negative linear combination of convex functions remains convex, and composing it with a convex non-decreasing activation preserves convexity. Therefore, all units in layer $l$ are convex. The final output is again a non-negative linear combination of convex functions, so $\Lambda^*(t)$ is convex. Thus $\Lambda^{*\prime\prime}(t)=\lambda^{*\prime}(t)\ge0$.
\end{proof}

\subsection{Proof of Theorem~\ref{thm: saturation limit}}
\begin{proof}
If $\sigma(x)\in[c_{\min},c_{\max}]$ and all weights and biases are finite, then the first hidden layer is bounded. Every subsequent pre-activation is a finite affine transformation of a bounded vector, and applying $\sigma$ keeps it bounded. By induction, $\boldsymbol{h}^{(L-1)}$ is bounded. For the standard positive-weighted MLP without an additional unbounded skip connection from time $t$ to the output, $\Lambda^*(t)=\boldsymbol{W}^{(L)}\boldsymbol{h}^{(L-1)}+\boldsymbol{b}^{(L)}$ is bounded by some finite $M$, giving $\limsup_{t\to\infty}\Lambda^*(t)\le M<\infty$.
\end{proof}

\subsection{Proof of Theorem~\ref{thm: accuracy of monotone splines}}

\label{app:proof of thm: accuracy of monotone splines}

\begin{proof}
We start from ${\Lambda}^{*\prime}_\theta(t)$. When the interpolation length $w^{(n)}_M$ is sufficiently large, only the interpolation function affects MAS. 
For common monotone splines, given the endpoint values and derivatives, the spline parameter can be uniquely defined. Then, for any interval $[a,b]\subset[t_n,t_{n+1}]$ for some timestamp $t_n$, set
\begin{align*}
    &{\Lambda}^{*\prime}_\theta(a)={\Lambda}^{*\prime}(a),
    {\Lambda}^*_\theta(a)={\Lambda}^{*}(a), \\
    &{\Lambda}^{*\prime}_\theta(b)={\Lambda}^{*\prime}(b),
    {\Lambda}^*_\theta(b)={\Lambda}^{*}(b).
\end{align*}
Given \cref{assum: C-1 for der}, both derivatives are Lipschitz continuous on $[a,b]$. Thus, for any $t\in[a,b]$, the following inequality holds:
\begin{align*}
    \vert{\Lambda}^{*\prime}_\theta(t)-{\Lambda}^{*\prime}(t)\vert
    &\le
    \vert{\Lambda}^{*\prime}_\theta(t)-{\Lambda}^{*\prime}_\theta(a)\vert
    +
    \vert{\Lambda}^{*\prime}(t)-{\Lambda}^{*\prime}(a)\vert \\
    &\le C_0\vert t-a\vert,
\end{align*}
where $C_0$ absorbs the Lipschitz constants. Note that the maximum interpolation length is $\Delta$. The above inequality implies
\[
    \Vert {\Lambda}^{*\prime}_\theta-{\Lambda}^{*\prime}\Vert_0\le C_0\Delta.
\]
Meanwhile,
\begin{align*}
    \vert\Lambda^{*}_\theta(t)-{\Lambda}^*(t)\vert
    &=\left|\int_a^t\left({\Lambda}^{*\prime}(w)-\Lambda^{*\prime}_\theta(w)\right)d w\right|\\
    &\le\int_a^t\vert{\Lambda}^{*\prime}(w)-\Lambda^{*\prime}_\theta(w)\vert d w\\
    &\le\int_a^t C_0\vert w-a\vert d w\le \frac{1}{2}C_0\Delta^2.
\end{align*}
Thus, we derive the fitting error bound for $\Lambda^{*}_\theta(t)$ and $\Lambda^{*\prime}_\theta(t)$.
\end{proof}

\subsection{Proof of Theorem~\ref{thm: Gen bound for MAS}}

\label{app:proof of thm: Gen bound for MAS}
When proving the generalization bound of the MAS model, we further assume that the horizontal interval of MAS interpolation is a fixed value \(\Delta\). That is, after each timestamp \(t_n\), there are \(M\) interpolation intervals, with a total length of \(L = M\Delta\). Equivalently, the interpolation support contains \(M+1\) knots; this off-by-one convention only changes constants. For the next timestamp \(t_{n+1}\), there is a probability that it falls into one of the \(M\) interpolation intervals, or into the extrapolation interval \([t_n + M\Delta, \infty)\).
We denote the index of the interval into which \(t_{n+1}\) falls as \(r_{n+1}\), taking values in \(1, 2, \dots, M\).

We first prove \cref{lem: Lipschitz for MAS}:
\begin{proof}
    For each time series, the negative log-likelihood differs from the log-likelihood only by a sign. Therefore, it is sufficient to bound the absolute difference of the log-likelihood terms. The log-likelihood can be decomposed into the loss over the \(N, N\le N_0 \) timestamps:
\begin{equation*}
\begin{aligned}
        \log p\left( S \right) &=\sum_{n=0}^{N-1}{\log}\frac{d}{dt}\Lambda _{k_{n}}^{*}\left( t_{n+1} \right) -\Lambda ^*\left( T \right)\\
        &=\sum_{n=0}^{N-1}\left[{\log}\frac{d}{dt}\Lambda _{k_n}^{*}\left( t_{n+1} \right) -\left( \Lambda ^*\left( t_{n+1} \right) -\Lambda ^*\left( t_n \right) \right)\right].
\end{aligned}
\end{equation*}
where $t_{N+1}=T$. For simplicity, we denote the $n$-th summand by $\log p(t_n)$.

For different parameters $\theta$ and $\eta$, the difference between $\log p(t_n,\theta)$ and $\log p(t_n,\eta)$ can be written as:
\begin{align*}
&\ \ \ \ \ \Vert\log p(t_n,\theta)-\log p(t_n,\eta)\Vert\\
&= \Big| \left(\log \left( \Lambda^{*\prime}_\theta(t_n) \right) -\left( \Lambda^{*}_\theta(t_n) -\Lambda^{*}_\theta(t_{n-1}) \right)\right) \\
&\quad-\left(\log \left( \Lambda^{*\prime}_\eta(t_n) \right) -\left( \Lambda^{*}_\eta(t_n) -\Lambda^{*}_\eta(t_{n-1}) \right)\right) \Big|\\
&\le \underbrace{\left|  \log \left( \Lambda^{*\prime}_\theta(t_n) \right) -\log \left( \Lambda^{*\prime}_\eta(t_n) \right) \right|}_{T_1}+
\underbrace{\left| \left( \Lambda^{*}_\theta(t_n) -\Lambda^{*}_\theta(t_{n-1}) \right) 
-
\left( \Lambda^{*}_\eta(t_n) - \Lambda^{*}_\eta(t_{n-1}) \right) \right|}_{T_2}.
\end{align*}
We next discuss the bound for $T_1$ and $T_2$. For $T_1$, we utilize \cref{assum: positive for der} and the Lipschitz property in \cref{assum: C-1 for der}:
\begin{align*}
    T_1
    &=\left|  \log \left( \Lambda^{*\prime}_\theta(t_n) \right) -\log \left( \Lambda^{*\prime}_\eta(t_n) \right) \right|
    \le \frac{1}{v_1}\Vert\Lambda^{*\prime}_\theta(t_n) -\Lambda^{*\prime}_\eta(t_n) \Vert.
\end{align*}
Suppose $t_n$ falls in the $r_{n}$-th interval in MAS. Suppose the parameter $\theta=\{\theta_1,\theta_2,\cdots,\theta_M\}$. 
Each component $\theta_m$ is a sub-vector, deciding the height $\Lambda^{*}_\theta(t_{i}+(m-1)\Delta)$ and the derivative $\Lambda^{*\prime}_\theta(t_{i}+(m-1)\Delta)$ and further deciding the interpolation $f^{*(n)}_{m,\theta}$. Because $t_{n-1}+\Delta\cdot({r_{n}}-1)<t_n<t_{n-1}+\Delta\cdot r_{n}$, the above inequality can be written as:
\begin{align*}
    T_1
    &\le \frac{1}{v_1}\Vert\Lambda^{*\prime}_\theta(t_n) -\Lambda^{*\prime}_\eta(t_n) \Vert
    \le\frac{l_2}{v_1}\lVert \theta ^{r_{n}}-\eta ^{r_{n}}\Vert.
\end{align*}
We next bound $T_2$. Note that,
\begin{align*}
\Lambda ^*_{\theta}\left( t_n \right)-\Lambda ^*_{\theta}\left( t_{n-1} \right) &=\left( \Lambda _{\theta}^{*}\left( t_{n-1}+\Delta \right) -\Lambda _{\theta}^{*}\left( t_{n-1} \right) \right) \\
&+\left( \Lambda _{\theta}^{*}\left( t_{n-1}+2\Delta \right) -\Lambda _{\theta}^{*}\left( t_{n-1}+\Delta \right) \right)  \\
&+\cdots+
\left( \Lambda _{\theta}^{*}\left( t_{n-1}+(r_{n}-1)\Delta \right) -\Lambda _{\theta}^{*}\left( t_{n-1}+(r_{n}-2)\Delta \right) \right) \\
&+\left( \Lambda _{\theta}^{*}\left( t_n \right) -\Lambda _{\theta}^{*}\left( t_{n-1}+(r_{n}-1)\Delta \right) \right).
\end{align*}
Thus, by the Lipschitz continuity of each active segment with respect to $\theta$,
\begin{align*}
    T_2
&=\left| \left( \Lambda^{*}_\theta(t_n) -\Lambda^{*}_\theta(t_{n-1}) \right)
-
\left( \Lambda^{*}_\eta(t_n) - \Lambda^{*}_\eta(t_{n-1}) \right) \right|\\
&\le \sum_{m=1}^{r_n}{\lVert \theta ^{m}-\eta ^{m} \rVert}+l_2 \lVert \theta ^{r_{n}}-\eta ^{r_{n}} \rVert.
\end{align*}
Combining the bound for $T_1$ and $T_2$, we obtain that:
\begin{align*}
    \Vert\log p(t_n,\theta)-\log p(t_n,\eta)\Vert
    &\le T_1+T_2\\
    &\le \frac{l_2}{v_1}\lVert \theta ^{r_{n}}-\eta ^{r_{n}}\Vert
    +\sum_{m=1}^{r_{n}}{\lVert \theta ^{m}-\eta ^{m} \rVert}
    +l_2 \lVert \theta ^{r_{n}}-\eta ^{r_{n}} \rVert\\
    &\le \max \left( \frac{l_2}{v_1}+l_2+1 ,1 \right) \sum_{m=1}^{r_{n}}{\lVert \theta ^m-\eta ^m \rVert}\\
&\le C_1 r_{n}\lVert \theta -\eta \rVert.
\end{align*}
By summing up $\log p(t_n)$ and taking average over time $T$, we can acquire:
\begin{equation*}
        \left\Vert (-\log p_{\theta}(S)) - (-\log p_{\eta}(S)) \right\Vert
        =
        \left\Vert \log p_{\theta}(S) - \log p_{\eta}(S) \right\Vert
        \le \frac{C_1}{T} \sum_{n=1}^N r_n \left\Vert \theta - \eta \right\Vert.
    \end{equation*}
This finishes the proof.
\end{proof}

Before proving \cref{thm: Gen bound for MAS}, we introduce the following concepts.
\begin{definition}[bracketing number \citep{van2000asymptotic}]
Given two functions $l$ and $u$, the \textit{bracket} $[l, u]$ is the set of all functions $f$ with $l \leq f \leq u$. Given a probability measure $\mu$, 
an $\varepsilon$-\textit{bracket} is a bracket $[l, u]$ with $\Vert u - l\Vert_{L_2}=\left(\int (u-l)^2\text{d}\mu\right)^{\frac{1}{2}} < \varepsilon$. The \textit{bracketing number} 
$N_{[\,]}(\varepsilon, \mathcal{F})$
is the minimum number of $\varepsilon$-brackets needed to cover $\mathcal{F}$. Specifically, 
\begin{align}
\label{def: bracketing number}
    &N_{[\,]}(\epsilon, \mathcal{F})=\min\{n:\exists\left[l_j,u_j\right]_{j=1}^{n}, \Vert u_j-l_j\Vert \le\epsilon, \notag \\
    &\hspace{2.5cm}\text{s.t. }  \forall f\in \mathcal{F},\exists j\in [n], l_j \le f\le u_j\}. 
\end{align}
\end{definition}

Bracketing numbers quantify how many simple function pairs (brackets) are needed to approximate a complex function class within a given accuracy. The smaller the bracketing number, the simpler or more regular the function class is. Utilizing the following lemma in \citep{van2000asymptotic}, we can acquire the bracketing number bound for MAS:
\begin{lemma}
    Let \( \mathcal{F} = \{ f_\theta : \theta \in \Theta \} \) be a collection of measurable functions indexed by a bounded subset \( \Theta \subset \mathbb{R}^d \). Suppose that there exists a measurable function \( m \) such that
\[
|f_{\theta_1}(x) - f_{\theta_2}(x)| \leq m(x) \|\theta_1 - \theta_2\|, \quad \text{for every } \theta_1, \theta_2.
\]

Then there exists a constant \( K \), depending on \( \Theta \) and \( d \) only, such that the bracketing numbers satisfy
\begin{align}
    &N_{[\ ]}(\varepsilon, \mathcal{F}) \leq K \left( \frac{\Vert m \Vert_{L_2} \operatorname{diam} \Theta}{\varepsilon} \right)^d \quad \notag\text{for every } 0 < \varepsilon < \operatorname{diam} \Theta, 
\end{align}
\label{lem: bracketing number}
where $\operatorname{diam}\Theta$ is the diameter of the parameter space $\Theta$.
\end{lemma}

Utilizing \cref{lem: bracketing number} and \cref{lem: Lipschitz for MAS}, the bracketing number of MAS can be upper bounded:

\begin{theorem}
    \label{thm: bracketing number of MAS}
    The bracketing number of MAS is upper bounded. Suppose a set of MAS $\mathcal{M}$ is controlled by $d$-dimensional parameter $\theta\in\Theta\subset \mathbb{R}^d$. For any $0<\epsilon<\operatorname{diam} \Theta$,
    \begin{equation}
        N_{[\ ]}(\varepsilon, \mathcal{M}) \leq K \left( \frac{c_1 \operatorname{diam} \Theta}{\varepsilon} \right)^d,
    \end{equation}
    where $c_1=\frac{B}{T}\left(\mathbb{E}\left(\sum_{n=1}^{N_0} r_n\right)^2\right)^{\frac{1}{2}}$.
\end{theorem}
An important feature of the bracketing number is that it bounds the complexity of a model. We first introduce the concept of Rademacher complexity \citep{Rademacher}.

\begin{definition}[Rademacher complexity \citep{Rademacher}]
    The \textit{Rademacher complexity} of a class of functions $\mathcal{F}$ with respect to a sample $\{x_1, x_2, \dots, x_Z\}$ drawn from a distribution $P$ is defined as:
    \[
    \mathcal{R}_Z(\mathcal{F}) = \mathbb{E}_{\sigma} \left[ \sup_{f \in \mathcal{F}} \frac{1}{Z} \sum_{z=1}^Z \sigma_z f(x_z) \right],
    \]
    where $\sigma_1, \sigma_2, \dots, \sigma_Z$ are independent Rademacher random variables, i.e., $\mathbb{P}(\sigma_z = 1) = \mathbb{P}(\sigma_z = -1) = \frac{1}{2}$, and the expectation $\mathbb{E}_{\sigma}$ is taken over the distribution of these random variables.
\end{definition}

Rademacher complexity can be understood as the degree to which a model fits noise. The higher the Rademacher complexity of a model, the stronger its ability to fit noise, and the more prone it is to overfitting. Since \cref{assum: bounded loss} bounds the negative log-likelihood class, we may normalize the loss functions into $[0,1]$; the scaling constant is absorbed into $C_2$ and $C_3$. We can derive the following lemma:

\begin{lemma}
\label{lem:Rademacher bound}
    For a real-valued function space $\mathcal{F}: \mathcal{X} \to [0,1]$, given a training set $X = \{x_1, x_2, \dots, x_Z\}$ of size $Z$ sampled independently and identically distributed from $\mathcal{D}$, for any $f \in \mathcal{F}$ and $0 < \xi < 1$, with probability at least $1 - \xi$, we have:
    \[
    \left[\mathbb{E}[f(x)]-\frac{1}{Z} \sum_{z=1}^Z f(x_z) \right] \leq   2 \mathcal{R}_Z(\mathcal{F}) + \sqrt{\frac{\ln(1/\xi)}{2Z}}.
    \]
\end{lemma}

There exists the following relationship between the Rademacher complexity and the bracketing number of the model:

\begin{theorem}[Bounding Rademacher Complexity via Bracketing Number]
\label{thm: brac num and rad comp}
Let \( \mathcal{F} \) be a class of measurable functions, and let \( P \) be a probability distribution over \( \mathcal{X} \). Suppose that the bracketing number \( N_{[\,]}(\varepsilon, \mathcal{F}) \) is finite for all \( \varepsilon > 0 \). Then, for a sample \( S = \{x_1, \ldots, x_Z\} \) drawn i.i.d. from probability measure \( \mu \), the empirical Rademacher complexity satisfies:
\begin{align*}
    \mathcal{R}_Z(\mathcal{F}) &\leq \inf_{\delta > 0} \left( 4\delta + \frac{12}{\sqrt{Z}} \int_\delta^{\operatorname{diam}\Theta} \sqrt{\log N_{[\,]}(\varepsilon, \mathcal{F})} \, \text{d}\varepsilon \right)\\
    &\le \frac{12}{\sqrt{Z}} \int_0^{\operatorname{diam}\Theta} \sqrt{\log N_{[\,]}(\varepsilon, \mathcal{F})} \, \text{d}\varepsilon.
\end{align*}
\end{theorem}

\begin{proof}
Let \( \sigma_1, \ldots, \sigma_Z \) be i.i.d. Rademacher random variables. Then:
\[
\mathfrak{R}_Z(\mathcal{F}) = \mathbb{E}_\sigma \left[ \sup_{f \in \mathcal{F}} \frac{1}{Z} \sum_{z=1}^Z \sigma_z f(x_z) \right].
\]
Let \( \delta > 0 \) be arbitrary. For each \( \varepsilon \in (\delta,1] \), construct an \( \varepsilon \)-bracketing cover \( \{[l_j, u_j]\} \) of \( \mathcal{F} \) such that \( \|u_j - l_j\|_{L_2(\mu)} \leq \varepsilon \).

For each \( f \in \mathcal{F} \), choose bracket midpoint \( f_\varepsilon = (l_j + u_j)/2 \). Then:
\[
\|f - f_\varepsilon\|_{L_2(\mu)} \leq \frac{\varepsilon}{2}.
\]
All these $f_\epsilon$ form a new hypothesis set $\mathcal{F}_\epsilon$. 
The size of $\mathcal F_\epsilon$ is smaller than its bracketing number, i.e., $\#\{\mathcal F_\epsilon\}\le{}N_{[\ ]}(\epsilon,\mathcal F)$.

By triangle inequality and properties of empirical norms:
\begin{equation*}
    \begin{aligned}
    \mathfrak{R}_Z(\mathcal{F})&=\mathbb{E}_\sigma \left[ \sup_{f \in \mathcal{F}} \frac{1}{Z} \sum_{z=1}^Z \sigma_z f(x_z) \right]\\
&\le\mathbb{E}_\sigma \left[ \sup_{f \in \mathcal{F}} \frac{1}{Z} \sum_{z=1}^Z  \sigma_z (f(x_z)-f_\epsilon(x_z))+\sigma_zf_\epsilon(x_z) \right]\\
&\leq \mathbb{E}_\sigma \left[ \sup_{f \in \mathcal{F}} \frac{1}{Z} \sum_{z=1}^Z  \sigma_zf_\epsilon(x_z) \right]+ \varepsilon \le\mathfrak{R}_Z(\mathcal{F}_\varepsilon)+\epsilon.
    \end{aligned}
\end{equation*}
Using Massart's Lemma \citep{Rademacher}:
\[
\mathfrak{R}_Z(\mathcal{F}_\varepsilon) \leq \sqrt{ \frac{2 \log N_{[\,]}(\varepsilon, \mathcal{F})}{Z} }.
\]
Integrating over \( \varepsilon \in [\delta,1] \) yields:
\[
\mathfrak{R}_Z(\mathcal{F}) \leq 4\delta + \frac{12}{\sqrt{Z}} \int_\delta^1 \sqrt{ \log N_{[\,]}(\varepsilon, \mathcal{F}) } \, \text{d}\varepsilon.
\]
Taking the infimum over \( \delta > 0 \) completes the proof. Setting $\delta=0$, we derive the second inequality in \cref{thm: brac num and rad comp}.
\end{proof}
Now that we get:
\begin{align*}
    \vert \widehat{\mathcal{L}}-\mathbb{E}[\widehat{\mathcal{L}}]\vert &\le 2 \mathcal{R}_Z(\mathcal{F}) + \sqrt{\frac{\ln(1/\xi)}{2Z}} \\
    &\le \frac{12}{\sqrt{Z}} \int_0^{\operatorname{diam}\Theta} \sqrt{\log N_{[\,]}(\varepsilon, \mathcal{F})} \, \text{d}\varepsilon+\sqrt{\frac{\ln(1/\xi)}{2Z}}\\
     &\le \frac{12}{\sqrt{Z}} \int_0^{\operatorname{diam}\Theta} \sqrt{\log \left[ K \frac{B}{T}W^{\frac{1}{2}d} \left( \frac{ \operatorname{diam} \Theta}{\varepsilon} \right)^d\right]} \, \text{d}\varepsilon + \sqrt{\frac{\ln(1/\xi)}{2Z}}\\
    &\le \frac{12}{\sqrt{Z}}\int_0^{\operatorname{diam}\Theta}{\sqrt{d}\sqrt{C_3 -\log \epsilon}}\,\text{d}\varepsilon +\sqrt{\frac{\ln \left( 1/\xi \right)}{2Z}},
\end{align*}
where $W=\mathbb{E}\left(\sum_{n=1}^{N} r_n\right)^2$.

Note that the parameter dimension is $N_0Mp$ ($N_0$ is the maximum number of timestamps), $M$ is the number of interpolation knots, and $p$ is the number of parameters required for each interpolation interval. In \cref{thm: Gen bound for MAS}, $p$ is treated as a fixed per-knot degree of freedom and absorbed into $C_2$; it is restored explicitly in the hyperparameter-selection corollary. By collecting each term in the above inequality, we acquire the final bound:
\begin{align*}
\vert \widehat{\mathcal{L}}-\mathbb{E}[\widehat{\mathcal{L}}]\vert \le\frac{1}{\sqrt{Z}}\left(\frac{1}{2}\sqrt{{\log\frac1\xi}}+{C_2\sqrt{N_0M}}\int_0^{c}\sqrt{C_3-\log t}d t\right),
\end{align*}
where $c=\operatorname{diam}\Theta$. This finishes the proof.

\subsection{Proof of Theorem~\ref{thm: Gen Error for MAS: 2}}
Suppose $\mathcal{L}^*$ is the optimal negative log-likelihood, derived from the ground-truth intensity. Then, the generalized loss for MAS can be decomposed as:
\begin{equation*}
    \begin{aligned}
        \mathbb{E}\widehat{\mathcal{L}}&\le \mathcal{L}^*+\left| \mathcal{L}^*-\widehat{\mathcal{L}} \right|+\left| \widehat{\mathcal{L}}-\mathbb{E}\widehat{\mathcal{L}} \right|
\\
&\le \mathcal{L}^*+\frac{1}{ZT}\sum_{z=1}^Z{\left| \log p^*\left( S_z \right) -\log p_{\theta}\left( S_z \right) \right|}+\left| \widehat{\mathcal{L}}-\mathbb{E}\widehat{\mathcal{L}} \right|\\
&\le \mathcal{L}^*
+
\frac{1}{ZT}\sum_{z=1}^Z{\sum_{n=1}^N{\left| \log p^*\left( t^{\left( z \right)}_n \right) -\log p_{\theta}\left( t_{n}^{\left( z \right)} \right) \right|}}+\left| \widehat{\mathcal{L}}-\mathbb{E}\widehat{\mathcal{L}} \right|\\
&\le 
\underbrace{\frac{1}{ZT}\sum_{z,n=1}^{Z,N}{\left| \log p^*\left( t_{n}^{\left( z \right)} \right) -\log p_{\theta}\left( t_{n}^{\left( z \right)} \right) \right|}\mathcal{I}\left( t_{n}^{\left( z \right)}\ge t_{n-1}^{\left( z \right)}+M\Delta \right)}_{W_1} \\
&+\underbrace{\frac{1}{ZT}\sum_{z,n=1}^{Z,N}{\left| \log p^*\left( t_{n}^{\left( z \right)} \right) -\log p_{\theta}\left( t_{n}^{\left( z \right)} \right) \right|}\mathcal{I}\left( t_{n}^{\left( z \right)}<t_{n-1}^{\left( z \right)}+M\Delta \right)}_{W_2} +\mathcal{L}^*
+
\underbrace{\left| \widehat{\mathcal{L}}-\mathbb{E}\widehat{\mathcal{L}} \right|}_{W_3}.
\end{aligned}
\end{equation*}
$W_1$ can be bounded by following inequality:
\begin{align*}
    \quad \ \ W_1&=\frac{1}{ZT}\sum_{z,n=1}^{Z,N}{\left| \log p^*\left( t_{n}^{\left( z \right)} \right) -\log p_{\theta}\left( t_{n}^{\left( z \right)} \right) \right|\mathcal{I}_n} \\
    &\le \frac{B_0}{T}\mathcal{I}\left( t_{n}^{\left( z \right)}-t_{n-1}^{\left( z \right)}\ge M\Delta \right), 
\end{align*}
where $\mathcal{I}_n=\mathcal{I}\left( t_{n}^{\left( z \right)}\ge t_{n-1}^{\left( z \right)}+M\Delta \right)$, and $B_0$ is a bound on $\left| \log p^*\left( t_{n}^{\left( z \right)} \right) -\log p_{\theta}\left( t_{n}^{\left( z \right)} \right) \right|$ from \cref{assum: bounded loss}. The factor depending on the maximum sequence length is absorbed into the constant $B$. Thus, we can take the expectation on both sides:
\begin{equation}
    \begin{aligned}
        &\quad \ \mathbb{E}W_1\le \frac{B_0}{T}\mathbb{P}\left( t_{n}^{\left( z \right)}-t_{n-1}^{\left( z \right)}\ge M\Delta \right) \\
        &\le \frac{B_0}{T}\int_{t_{n}^{\left( z \right)}-t_{n-1}^{\left( z \right)}\ge M\Delta}{\frac{t_{n}^{\left( z \right)}-t_{n-1}^{\left( z \right)}}{M\Delta}}\text{d}\mu \le \frac{B_0}{T}\frac{\mathbb{E}t_i}{L}:=\frac{B}{TL}.
    \end{aligned}
\end{equation}
For $W_2$, we have proven in \cref{thm: accuracy of monotone splines} that there exists a MAS $\Lambda^*_\theta$, s.t., $\Vert\Lambda^*_\theta-\Lambda^*\Vert\le\frac{1}{2}C_0\Delta^2$, $\Vert\Lambda^{*\prime}_\theta-\Lambda^{*\prime}\Vert\le C_0\Delta$.
As a result, $W_2$ can be bounded by:
\begin{equation}
    \begin{aligned}
        \ \ W_2&=\frac{1}{ZT}\sum_{z=1}^Z{\sum_{n=1}^N{\left| \log p^*\left( t_{n}^{\left( z \right)} \right) -\log p_{\theta}\left( t_{n}^{\left( z \right)} \right) \right|}}(1-\mathcal{I}_n) \\
&\le \frac{1}{ZT}\sum_{z=1}^Z{\sum_{n=1}^N{\left( \left| \log \Lambda _{\theta}^{*'}\left( t_{n}^{\left( z \right)} \right) -\log \Lambda ^{*'}\left( t_{n}^{\left( z \right)} \right) \right|\right.}}\\
&+\left.\left| \left( \Lambda _{\theta}^{*}\left( t_{n}^{\left( z \right)} \right) -\Lambda _{\theta}^{*}\left( t_{n-1}^{\left( z \right)} \right) \right) -\left( \Lambda ^{*}\left( t_{n}^{\left( z \right)} \right) -\Lambda ^{*}\left( t_{n-1}^{\left( z \right)} \right) \right) \right| \right)\\
&\le \frac{1}{ZT}\sum_{z=1}^Z{\sum_{n=1}^N{\left( \frac{1}{v_1}\lVert \Lambda _{\theta}^{*'}\left( t_{n}^{\left( z \right)} \right) -\Lambda ^{*'}\left( t_{n}^{\left( z \right)} \right) \rVert +C_0\Delta ^2 \right)}}\\
&\le \frac{1}{ZT}\sum_{z=1}^Z{\sum_{n=1}^N{\left( \frac{C_0}{v_1}\Delta +C_0\Delta ^2 \right)}}
\le\frac{1}{T}{\sum_{n=1}^{N_0}{\left( \frac{C_0}{v_1}\Delta +C_0\Delta ^2 \right)}}:={C_0\left(\frac{\Delta}{C_4}+\frac{\Delta^2}{2}\right)}.
    \end{aligned}
\end{equation}
Here $C_4$ absorbs the lower intensity bound $v_1$ and the normalization constants involving $T$ and $N_0$.
Finally, according to \cref{thm: Gen bound for MAS}:
\begin{align*}
    W_3&=\vert \widehat{\mathcal{L}}-\mathbb{E}[\widehat{\mathcal{L}}]\vert 
    \le
    \frac{1}{\sqrt{Z}}\left(\frac{1}{2}\sqrt{{\log\frac1\xi}}\right)+
    \frac{1}{\sqrt{Z}}\left({C_2\sqrt{N_0M}}\int_0^{c}\sqrt{C_3-\log t}d t\right)\\
    :&=\frac{1}{\sqrt{Z}}\left(\frac{1}{2}\sqrt{{\log\frac1\xi}}\right)+\frac{C_2}{\sqrt{Z}}\frac{\sqrt{L}}{\sqrt \Delta}R(L,\Delta).
\end{align*}
By combining $W_1,W_2$ and $W_3$, we acquire:
\begin{equation}
\label{equ: bound (app)}
    \mathbb{E}[\widehat{\mathcal{L}}]\le
    {\mathcal{L}^*}+
    {\frac{1}{2\sqrt{Z}}\sqrt{{\log\frac1\xi}}}+{\frac{B}{T\cdot L}}
    +{C_0\left(\frac{\Delta}{C_4}+\frac{\Delta^2}{2}\right)}+
    {\frac{C_2}{\sqrt{Z}}\frac{\sqrt{L}}{\sqrt{\Delta}}R(L,\Delta)}.
\end{equation}
This finishes the proof. 

\subsection{Proof of Theorem~\ref{thm: MAS vs MNN gap}}

\begin{proof}
Let $a=t_n^+$ and $b=t_n+T_w$. By assumption,
\[
    \lambda^*(a)-\lambda^*(b)=\gamma>0.
\]
For any convex MNN in \cref{thm: convexity trap}, the induced intensity is non-decreasing on the interval, so
\[
    \lambda_{\mathrm{MNN}}(b)\ge \lambda_{\mathrm{MNN}}(a).
\]
If the MNN approximation error were smaller than $\gamma/2$ at both endpoints, then
\[
    \lambda_{\mathrm{MNN}}(a)>\lambda^*(a)-\frac{\gamma}{2}
    =
    \lambda^*(b)+\frac{\gamma}{2},
\]
and
\[
    \lambda_{\mathrm{MNN}}(b)<\lambda^*(b)+\frac{\gamma}{2}.
\]
These two inequalities imply $\lambda_{\mathrm{MNN}}(a)>\lambda_{\mathrm{MNN}}(b)$, contradicting the non-decreasing property. Therefore,
\[
    \inf_{\mathrm{MNN}}\mathrm{Error}(\lambda_{\mathrm{MNN}})\ge \frac{\gamma}{2}.
\]
On the other hand, \cref{thm: accuracy of monotone splines} gives an MAS intensity approximation error bounded by $C_0\Delta$. Hence,
\[
    \inf_{\mathrm{MNN}}\mathrm{Error}(\lambda_{\mathrm{MNN}})
    -
    \mathrm{Error}(\lambda_{\mathrm{MAS}})
    \ge
    \frac{\gamma}{2}-C_0\Delta,
\]
which proves the theorem.
\end{proof}

\subsection{Proof of Corollary~\ref{col: selection of L and delta}}

Finally, we estimate the bounds for two important parameters, $\Delta$ and $L$, by minimizing the generalization bound in \cref{thm: Gen Error for MAS: 2}. If we restore all constants in the bound \cref{equ: bound (app)}, we get:

\begin{equation}
\label{equ: bound detail}
    \mathbb{E}[\widehat{\mathcal{L}}]\le
    {\mathcal{L}^*}
    +
    {\frac{1}{2\sqrt{Z}}\sqrt{{\log\frac1\xi}}}
    +
    {\frac{N_0B_0}{p}\cdot\frac{\mathbb{E} t_i}{T\cdot L}}
    +
    {l_1\left(\frac{L}{pv_1}+\frac{L^2}{p^2}\right)}
    +
    {\frac{12}{\sqrt{Z}}{\sqrt{N_0pq}}\int_0^{\text{diam}\Theta}\sqrt{C_3-\log t}dt},
\end{equation}

where $\Delta=L/p$. We first focus on how to estimate $L$. Typically, $\Delta\ll1$, thus we could ignore the $\Delta^2$ term in \cref{equ: bound (app)} and focus on the term $\frac{l_1L}{v_1p}$. Notice that in \cref{equ: bound detail}, only two terms ${N_0B_0\cdot\frac{\mathbb{E} t_i}{T\cdot L}}+\frac{l_1L}{v_1p}$ relate to $L$. Then when \cref{equ: bound detail} is minimized, $L$ should satisfy:
\[
L\ge\sqrt{\frac{v_1N_0B_0\mathbb{E}[t_i]}{l_1T}}.
\]
As a bound between the ground truth $\log p^*\left( t_{n}^{\left( z \right)} \right)$ and  $\log p_{\theta}\left( t_{n}^{\left( z \right)} \right)$, $B_0=\sup_{t^{(z)}_n}\left| \log p^*\left( t_{n}^{\left( z \right)} \right) -\log p_{\theta}\left( t_{n}^{\left( z \right)} \right) \right|>1$ typically holds. Moreover, $N_0$ is the maximum number of events that happen during the interval $[0,T]$. Thus, $N_0B_0\ge T $, and the bound can be reduced to:
\[
L\ge\sqrt{\frac{v_1\mathbb{E}[t_i]}{l_1}},
\]
where $l_1$ is the Lipschitz constant of the intensity during event intervals, and $v_1$ is the base intensity during $[0,T]$.

We next focus on deriving the bound of $p$. Because $\Delta=L/p$, \cref{equ: bound detail} can be rewritten as:

\begin{align}
\label{equ: bound detail 2}
    \mathbb{E}[\widehat{\mathcal{L}}] 
    & \le
    {\mathcal{L}^*}
    +
    {\frac{1}{2\sqrt{Z}}\sqrt{{\log\frac1\xi}}}
    +
    {\frac{N_0B_0}{p}\cdot\frac{\mathbb{E} t_i}{T\cdot \Delta}}
    +
    {l_1\left(\frac{L}{pv_1}+\frac{L^2}{p^2}\right)} \notag \\
    &+
    {\frac{12}{\sqrt{Z}}{\sqrt{N_0pq}}\int_0^{\text{diam}\Theta}\sqrt{C_3-\log t}dt}, \notag \\
    &\le
    {\mathcal{L}^*}
    +
    {\frac{1}{2\sqrt{Z}}\sqrt{{\log\frac1\xi}}}
    +
    {\frac{2N_0B_0}{p}\cdot\frac{\mathbb{E} t_i}{T\cdot \Delta}}
    +
    {\frac{12}{\sqrt{Z}}{\sqrt{N_0pq}}\int_0^{\text{diam}\Theta}\sqrt{C_3-\log t}dt}.
\end{align}

The second step is due to the assumption that $l_1\Delta^2$ is significantly smaller than $\Delta$ and can be ignored. 
Let $C_5=\int_0^{\text{diam}\Theta}\sqrt{C_3-\log t}dt$. For a fixed interpolation length $L$, the leading $p$-dependent terms are
\[
    \frac{l_1L}{v_1p}
    +
    \frac{12C_5}{\sqrt{Z}}\sqrt{N_0pq}.
\]
Minimizing the above expression equals the following convex optimization regarding $p$: $\frac{A}{p}+B\sqrt{p}$, $A,B>0$, where $A=\frac{l_1L}{v_1}$ and $B=\frac{12C_5\sqrt{N_0q}}{\sqrt{Z}}$. $f(p)=\frac{A}{p}+B\sqrt{p}$ is minimized when $f'(p)=0$, which implies $p=\left(\frac{2A}{B}\right)^{\frac{2}{3}}$. That is,

\begin{align}
    p^\star
    &=
    \left(
    \frac{l_1L\sqrt{Z}}
    {6v_1C_5\sqrt{N_0q}}
    \right)^{\frac{2}{3}}.
\end{align}

Thus, the optimal number of interpolation intervals increases with the sample size $Z$ and decreases with the per-knot freedom $q$ and maximum sequence length $N_0$. When $L$ is set to the lower-bound scale derived above and $v_1$ is identified with the baseline intensity $\lambda_{\min}$, because $\Delta=L/p^\star$, we acquire the corresponding interval scale:

\[
\Delta
=\frac{L}{p^\star}
\le
C_{\Delta}
\frac{(N_0q)^{\frac{1}{3}}\lambda_{\text{min}}^{\frac{5}{6}}(\mathbb E [t_i])^{\frac{1}{6}}}
{Z^{\frac{1}{3}}l_1^{\frac{5}{6}}},
\]
where $C_{\Delta}$ absorbs the numerical and bracketing-integral constants such as $C_5$. This finishes the proof.

}

\vfill

\end{document}